\documentclass{article} 
\usepackage{iclr2026_conference,times}

\usepackage{amsmath,amsfonts,bm}

\def\eqref#1{equation~\ref{#1}}

\def\1{\bm{1}}

\DeclareMathAlphabet{\mathsfit}{\encodingdefault}{\sfdefault}{m}{sl}
\SetMathAlphabet{\mathsfit}{bold}{\encodingdefault}{\sfdefault}{bx}{n}

\DeclareMathOperator*{\argmin}{arg\,min}

\usepackage{hyperref}
\usepackage{url}

\usepackage{hyperref}       
\usepackage{url}            
\usepackage{booktabs}       
\usepackage{amsfonts}       
\usepackage{nicefrac}       
\usepackage{microtype}      
\usepackage{xcolor}         
\usepackage{tikz}
\usepackage{graphicx}%

\usepackage{multirow} 
\usepackage{subcaption}
\usepackage{amsmath}
\usepackage{amssymb}
\usepackage{amsthm}
\usepackage{mathtools}
\usepackage{enumitem}
\usepackage{wrapfig}
\usepackage{bold-extra}
\usepackage[capitalize,noabbrev]{cleveref}

\usepackage{thmtools,thm-restate}
\theoremstyle{plain}
\newtheorem{assumption}{Assumption}

\newtheorem{definition}{Definition}

\newtheorem{remark}{Remark}

\usepackage{algorithm}
\usepackage[noend]{algorithmic}

\newcommand{\purechildren}{\text{PCh}_{\graph}}
\newcommand{\node}[1]{\mathsf{#1}}
\newcommand{\set}[1]{\mathbf{#1}}
\newcommand{\setset}[1]{\mathcal{#1}}
\newcommand{\graph}{\mathcal{G}}
\newcommand{\parents}{\text{Pa}_{\graph}}

\newcommand{\empX}{\hat{\Sigma}_{\set{X}}}

\newcommand{\operator}{\mathcal{O}}

\newcommand{\setstate}{\mathbb{S}}
\newcommand{\state}{\mathcal{S}}

\DeclareMathOperator{\supp}{supp}
\DeclareMathOperator{\diag}{diag}

\DeclareMathOperator{\rank}{rank}

\crefname{assumption}{Assumption}{Assumptions}
\usepackage{thmtools,thm-restate}
\usepackage{enumitem}

\title{Score-based Greedy Search for Structure Identification of Partially Observed Causal Models}

\author{%
    Xinshuai Dong$^{1}$ \quad
    Ignavier Ng $^{1}$  \quad
    Haoyue Dai$^{1}$ \quad
    Jiaqi Sun$^{1}$ \quad
    Xiangchen Song$^{1}$ \quad \\
    \textbf{Peter Spirtes}$^{1}$ \quad
    \textbf{Kun Zhang}$^{1,2}$\\
    $^1$Carnegie Mellon University\\
    $^2$Mohamed bin Zayed University of Artificial Intelligence
}

\iclrfinalcopy 
\begin{document}

\maketitle

\begin{abstract}
Identifying the structure of a partially observed causal system 
is essential to various scientific fields.
Recent advances have focused on constraint-based causal discovery to solve this problem,
and yet in practice these methods often face challenges related to multiple testing and error propagation.
These issues could be mitigated by a score-based method
and thus it has  raised great attention whether there exists a score-based greedy search method that can handle the partially observed scenario.
In this work, 
we propose the first score-based greedy search method for the identification of structure involving latent variables
with
identifiability guarantees.
Specifically,
we propose {the} Generalized N Factor Model and establish {its} global consistency: 
 the true structure including latent variables can be identified  up to the Markov equivalence class by using score.
We then design
 Latent variable Greedy Equivalence Search (LGES),
a greedy search algorithm for this class of {models} with well-defined operators,
 which {searches} very efficiently over the graph space to find the optimal structure.
Our experiments on both synthetic and real-life data validate the  effectiveness of our method
(code will be available at \url{https://github.com/dongxinshuai/scm-identify}).
\end{abstract}

\vspace{-1em}
\section{Introduction and Related Work}
\label{sec:intro}
\vspace{-0.5em}
Causal discovery aims at 
identifying the causal relations
from  observational data
and it is crucial to many scientific fields \citep{spirtes2001causation,pearl2009causality}. 
However, traditional methods such as PC \citep{spirtes2001causation}, GES \citep{chickering2002optimal}, and LiNGAM \citep{shimizu2006linear},
rely on the causal sufficiency assumption, i.e., the absence of latent variables,
 which hardly holds in many real-world scenarios. 
  Therefore, extensive efforts are being made {towards} structure identification of a partially observed causal model. 
  
To handle this problem, the earliest attempts make use of conditional independence including Fast Causal Inference (FCI)~\citep{spirtes2001causation, zhang2008completeness} 
and its variants~\citep{colombo2012learning, spirtes2013causal, claassen2013learning, akbari2021recursive},
as well as over-complete ICA-based techniques~\citep{hoyer2008estimation,salehkaleybar2020learning}. 
Yet, these methods only focus on identifiable relations among observed variables, 
and the results provide limited information about structure among latent variables.
\looseness=-1


To this end, recent advance has been on {the} discovery of {the} entire structure including latent variables,
by introducing additional parametric or graphical assumptions.
 Typical  methods include rank or tetrad constraints based under linearity assumption~\citep{silva2003learning,silva2006learning,silva2005generalized,choi2011learning,kummerfeld2016causal,huang2022latent,dong2023versatile},
   high-order moments~\citep{shimizu2009estimation,zhang2018causal,cai2019triad,salehkaleybar2020learning,xie2020generalized,adams2021identification,
   dai2022independence,amendola2023third,wang2023causal},
    matrix decomposition~\citep{anandkumar2013learning}, mixture oracles~\citep{kivva2021learning},
     and multiple domains \citep{zeng2021causal,sturma2023unpaired}. 
     Despite of the asymptotic correctness, however,
     these methods generally fall into the category of constraint-based
     methods; they rely heavily on statistical {tests}
     to iteratively construct the structure and thus suffer from 
     the problem of multiple-testing and error propogation \citep{spirtes2010introduction,colombo2012learning},
     especially with small sample size and large number of variables.
     \looseness=-1

On the other hand, score-based causal discovery may not suffer from these issues 
and thus is believed to be more practially favorable
 \citep{nandy2018high,ramsey2017million}.
One of the most classical methods is Greedy Equivalence Search (GES)~\citep{chickering2002optimal},
and yet it cannot handle the existence of latent variables.
Later, various score-based methods that allow the existence of latent variables have been proposed
\citep{shpitser2012parameter,triantafillou2016score,nowzohour2017distributional,
bhattacharya2020differentiable,shahin2020automatic,bernstein2020ordering,bellot2021deconfounded,claassen2022greedy},
and yet they still focus only on relations among observed variables, except the one by \citet{zhang2004hierarchical} without identifiability and another with exact search by \citet{ngscore} {(a more detailed discussion between this work and the exact search \citep{ngscore} can be found in \cref{sec:relatedwork})}.
%
As a consequence, {the following crucial question naturally arises}:
Is it possible to develop a score-based greedy search method that can efficiently recover the entire underlying structure 
involving latent variables with an asymptotic correctness guarantee?\looseness=-1

To address such a challenging problem, 
we are confronted with three fundamental questions:
(i) What is  {the essential relations between structure identifiability and likelihood scores?}
(ii) What graphical assumptions are needed to uniquely recover the structure by using score?
(iii) How can we design an efficient algorithm to search over the graph space to find the optimal structure?
 We provide our answers to these questions respectively in \cref{sec:algebraic_equivalence,sec:graphical_assumption_and_consistency,sec:method}
 and our contributions can be summarized as follows.
 \looseness=-1


\begin{itemize}[leftmargin=*, itemsep=0pt]
\item We characterize how {the} likelihood score can be used for the structure identification of partially observed linear causal models. Specifically, we show that
      the structure with the best likelihood score and minimal dimension is algebraically equivalent to the ground truth (in \cref{thm:equivalence_equality_constraints}).
\item We propose the GNFM (in Def. \ref{definition:gnfm}),
and 
accordingly establish the global consistency
of score for it -
the whole underlying structure can be uniquely recovered up to the {bMarkov Equivalence Class (MEC)} by using {the score}
 (in \cref{thm:identify_gnfm,corollary:global_consistency}). 
This graphical condition is rather mild and takes the prevalent one factor model \citep{silva2003learning} as a special case.
\item 
 We develop the Latent variable Greedy Equivalence Search (LGES),
 an efficient causal discovery algorithm 
 for identifying structure involving latent variables.
 To our best knowledge, this is the first score-based greedy search with identifiability {guarantees}
  in the  partially observed scenario (in Thm. \ref{theorem:correctphase2}).
 Our experiments on both synthetic and
 real-world dataset empirically validate its effectiveness.
\end{itemize}

\vspace{-1em}
\section{Preliminaries}
 \vspace{-0.5em}
\subsection{Problem Setting}
\vspace{-0.5em}
We aim to identify the structure of a partially observed 
linear causal model, defined as follows.

\begin{definition} [Partially Observed Linear Causal Models]
    Let $\graph:=(\set{V}_\graph,\set{E}_\graph)$ be a Directed Acyclic Graph (DAG) and 
     variables follow a linear SEM as
$\set{V}_\graph=F^T\set{V}_\graph+\epsilon_{\set{V}_\graph}$,
  where $\set{V}_\graph=\set{L}_\graph\cup\set{X}_\graph=\{\node{V}_i\}_{1}^{m+n}=\{\node{L}_i\}_{1}^{m}\cup\{\node{X}_i\}_{1}^{n}$
   contains $m$ latent variables and $n$ observed variables,
   $F=(f_{j,i})$ 
   is the weighted adjacency matrix and 
   $f_{j,i}\neq 0$ if and only if $V_j$ is a parent of $V_i$ in $\graph$,
   and $\epsilon_{\node{V}_i}$ represents the Gaussian noise term of $\node{V}_i$.
  \label{definition:polcm}
  \vspace{-1mm}
  \end{definition}
  
  Our goal is to identify the underlying  structure $\graph$ over all the variables $\set{L}_\graph\cup\set{X}_\graph$,
  given i.i.d. samples of observed variables $\set{X}_\graph$ only. 
  Note that the name/order of latent variables can never be identified so 
  we focus on structure identification up to permutation of latent variables.
  Without loss of generality, we can assume that all variables have zero mean, and thus
  the observational data can also be summarized as the empirical covariance matrix over observed variables,
 i.e., $\hat{\Sigma}_{\set{X}_\graph}$.
 We use $\node{V}$ and $\set{V}$, to denote a random variable and a set of variables, respectively.
 We drop the subscript $\graph$ in $\set{L}_\graph$ and $\set{X}_\graph$ when the context is clear. 
 For a matrix $M$, we define its support set as $\supp(M)\coloneqq\{(i,j):M_{i,j}\neq 0\}$. 
 $\graph_1$ and $\graph_2$ belong to the same MEC iff they share the same skeleton and set of v-structures {(in \cref{def:vstructure})}
 over the entire graph including latent variables.

\vspace{-1em}
\subsection{Likelihood Score}
\vspace{-0.5em}
Despite the asymptotic correctness of constraint-based causal discovery approaches,
in practice these methods 
often suffer from the problem of multi-testing and error propagation \citep{spirtes2010introduction,colombo2012learning}.
In the finite sample case,
they rely heavily on statistical tests
to iteratively build the result, while the 
power of each test might be limited especially when the sample size is small and number of variables is large.
 %
 On the contrary,
 score-based causal discovery methods
 may not suffer from these problems and could be  
 practically more favorable \citep{nandy2018high,ramsey2017million},
 especially when the sample size is small (also empirically validated in \cref{sec:synthetic_data}).
 {A more detailed discussion about error propagation can be found in \cref{sec:discussion_on_error_propagation}.}

 %

 %
%
To this end, in this work we aim at structure identification based on the use of likelihood scores in the partially observed scenario.
In contrast to the fully observed case,
in the presence of latent variables, 
the formulation of the likelihood is not trivial,
and we provide it in what follows.

\begin{restatable}[Parameterization of Population Covariance \citep{dong2024parameter}]{proposition}{PropositionCovarianceMatrix}\label{proposition:covariance_matrix}
    Consider the model defined in Def.~\ref{definition:polcm},
    and let $F${\tiny$=\begin{pmatrix}
        F_{\set{L}\set{L}}   & F_{\set{L}\set{X}}  \\
        F_{\set{X}\set{L}}   & F_{\set{X}\set{X}} 
       \end{pmatrix}$}, and $\Omega${\tiny$=\begin{pmatrix}
        \Omega_{\epsilon_{\set{L}}} & 0 \\
        0 & \Omega_{\epsilon_{\set{X}}}
        \end{pmatrix}$},
        where $\Omega$ is the diagonal covariance matrix of $\epsilon_{\set{V}_\graph}$.
        Let
    $M=((I-F_{\set{L}\set{L}}-F_{\set{L}\set{X}}(I-F_{\set{X}\set{X}})^{-1} F_{\set{X}\set{L}}))^{-1}$, 
    $N=(((I-F_{\set{L}\set{L}})F_{\set{X}\set{L}}^{-1}(I-F_{\set{X}\set{X}})-F_{\set{L}\set{X}}))^{-1}$,
    and $\Sigma_{\set{L}}=M^{T} \Omega_{\epsilon_\set{L}} M + N^T\Omega_{\epsilon_\set{X}} N$.
    Then the population covariance of $\set{X}$ can be formulated as
    \vspace{-0.5em}
    \begin{equation}
    \Sigma_{\set{X}}=(I-F_{\set{X}\set{X}})^{-T}\Big(F_{\set{L}\set{X}}^{T} \Sigma_{\set{L}} F_{\set{L}\set{X}} 
    +\Omega_{\epsilon_\set{X}} N F_{\set{L}\set{X}} + \Omega_{\epsilon_\set{X}} +F_{\set{L}\set{X}}^{T} N^T \Omega_{\epsilon_\set{X}}\Big)(I-F_{\set{X}\set{X}})^{-1}.
    \end{equation}
    \vspace{-1.5em}
\end{restatable}

By making use of \cref{proposition:covariance_matrix} to parametrize $\Sigma_{\set{X}}$, the maximum log-likelihood of a given structure $\graph$ and 
observation $\empX$ is as follows {($\operatorname{tr}$ and $\det$ refers to matrix trace and determinant, respectively)}. 
\vspace{-0em}
\begin{align}
    \label{eq:ml}
    &\text{score}_{\text{ML}}(\graph,\empX)=\max_{\substack{(F,\Omega):~
    \supp(F)\subseteq \supp(F_\graph),
    \Omega\in\diag(\mathbb{R}_{> 0}^{n+m})}
    } ~\mathcal{L}, \\
    &\mathcal{L}=-(N/2)(
    \operatorname{tr}((\Sigma_{\set{X}})^{-1}\hat{\Sigma}_{\set{X}})+\log \det \Sigma_{\set{X}}).
\end{align}
\vspace{-1.5em}

We next show the theoretical foundation of using maximum likelihood score for the structure identification of partially observed linear causal models.

\vspace{-1em}
\section{Score-Based Identifiability Theory for Partially Observed Causal Models}
\vspace{-0.5em}
\subsection{Algebraic Equivalence by Score and Dimension}
\vspace{-0.5em}
\label{sec:algebraic_equivalence}

Consider a model in Def. \ref{definition:polcm}.
Its structure imposes various types of equality (i.e., algebraic) constraints on the covariance matrices (over observed variables),
no matter how its parameters $(F,\Omega)$ may change.
The imposed equality constraints are properties of the observational distribution and
 contain crucial graphical information about the underlying structure $\graph$.
Such constraints 
include conditional independence (i.e., vanishing partial correlation) 
constraints \citep{spirtes2001causation},
 rank constraints (i.e., vanishing determinant) \citep{spirtes2001causation,sullivant2010trek}, 
 and possibly Verma constraints~\citep{verma1991equivalence}. An overview can be found in \citet{drton2018algebraic}.

  Let $H(\graph)$ be the set of equality constraints imposed by structure $\graph$ on the generated 
  covariance matrices over observed variables (detailed in \cref{def:HG}),
  $\mathbb{G}^n$ be the set of all DAG structures that has $n$ measured variables,
  and 
  $\mathbb{H}^n\coloneqq\bigcup_{\graph\in\mathbb{G}^n}B(\graph)$,
  where $B(\graph)$ consists of the canonical and minimal set of equality constraints with respect to reduced Gröbner basis 
  (detailed in \cref{def:BGHn}).
  We say two structures $\graph_1$ and $\graph_2$ are { algebraically} equivalent,
   if they lead to the same equality constraints (on the observational distribution), i.e., $H(\graph_1)=H(\graph_2)$
   ~\citep{ommen2017algebraic}.
    Similar to the classical CI faithfulness assumption in causal discovery \citep{spirtes2001causation},
    to better relate the constraints to the underlying structure,
    we assume the generalized faithfulness as follows.

\begin{assumption}[Generalized faithfulness \citep{ghassami2020characterizing}]\label{assumption:generalized_faithfulness}
    A distribution $\Sigma_\set{X}$ is said to be generalized faithful to DAG $\graph \in \mathbb{G}^n$ if the entries of $\Sigma_\set{X}$ satisfy 
    an equality constraint $\kappa\in\mathbb{H}^n$ only if $\kappa\in H(\graph)$.
    \vspace{-0.5em}
    \end{assumption}

    In \cref{assumption:generalized_faithfulness},
    it suffices to use $\Sigma_{\set{X}}$ to denote the distribution, 
    as $\set{X}$ are jointly gaussian 
     and mean do not contain any information about structure \citep{ghassami2020characterizing}.
    Note that different types of faithfulness assumptions have been adopted in causal discovery,
    e.g., CI faithfulness and rank faithfulness \citep{spirtes2001causation,ghassami2020characterizing,huang2022latent,dong2023versatile}
    and 
    \cref{assumption:generalized_faithfulness} is the generalized version of them for linear causal models.
    Similar to CI faithfulness, generalized faithfulness is justified by that 
    the set of parameters that result in violation has Lebesgue measure zero 
    \citep{ghassami2020characterizing}
    and it has been  widely adopted in the field \citep{ng2020role,bhattacharya2021differentiable,sethuraman2023nodags}.
     On the other hand,
    without faithfulness, graphical information extracted from  observations cannot be trusted, 
    which makes structure identification extremely hard, if not impossible.

    Furthermore, let $\dim(\mathcal{G})$ denote the model dimension or degrees of freedom of DAG 
    $\mathcal{G}$ for the marginal over the observed variables (which can also be viewed as the number of free parameters of the set of distribution it can generate).
    In the absence of latent variables,  the degrees of freedom are nothing but the sum of { the} number of edges and {the} number measured variables.
    However,
     in the presence of latent variables, it does not necessarily hold \citep{geiger1996asymptotic,geiger2001stratified}.
    Without any specific graphical assumption, capturing the dimension could be highly non-trivial; e.g., the analysis of the degrees of freedom for sparse factor analysis,
    where latent variables are independent, already involves complex techniques from algebraic statistics \citep{drton2023algebraic}.
    We note that the focus of this work is not to characterize the dimension of each structure.
     In contrast, we only need to know the basic idea of dimension here and later we will show that
    a greedy search does not necessarily rely on knowing the exact dimension of a graph.

    \begin{figure}[t]
        \vspace{-0mm}
      \setlength{\belowcaptionskip}{0mm}
          \subfloat[Ground truth $\graph^*$.]
          {
          \centering
           \vspace{-1mm}\includegraphics[width=0.18\textwidth]{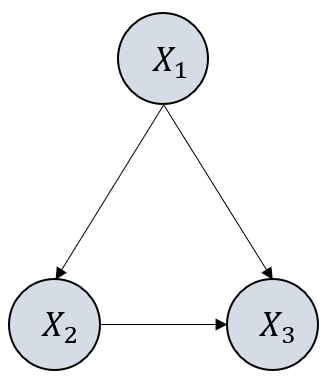}
          }
          \hspace{7mm}
          \subfloat[$\hat{\graph}_1$ by ~\cref{thm:equivalence_equality_constraints}.]
          {
          \centering
           \vspace{-1mm}\includegraphics[width=0.18\textwidth]{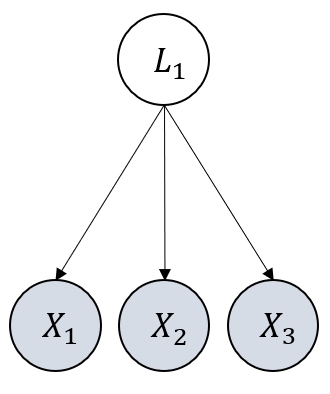}
          }
          \hspace{7mm}
          \subfloat[$\hat{\graph}_2$ by ~\cref{thm:equivalence_equality_constraints}.]
          {
          \centering
          \vspace{-1mm}
          \includegraphics[width=0.18\textwidth]{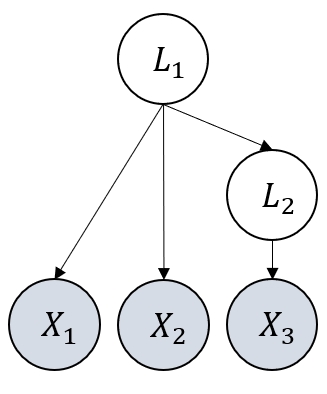}
          }
          \centering
          \vspace{-3mm}
      \caption{
      Without further graphical assumption,
      the algebraic equivalence class is very large and not very informative:
      suppose the ground truth $\graph^*$ in (a),
       by ~\cref{thm:equivalence_equality_constraints} we may arrive at either $\hat{\graph}_1$ (b) or $\hat{\graph}_2$ (c),
       both are algebraically equivalent to $\graph^*$. 
       }
        \label{fig:1}
        \vspace{-6mm}
      \end{figure}

    Now we  are ready to present the key result of this subsection, achieving algebraic equivalence by making use of {the likelihood} score,
    captured in the following theorem.

    \begin{restatable}[Algebraic Equivalence by Score and Dimension]{theorem}{TheoremEquivalenceEqualityConstraints}\label{thm:equivalence_equality_constraints}
        Suppose a model follows \cref{definition:polcm} with $\graph^*$
        and distribution $\Sigma^{*}_\set{X}$ satisfies the generalized faithfulness assumption.
        Given observation $\empX$
        and let $\mathbb{G}^* = \arg\max_{\graph\in\mathbb{G}^n}
        \text{score}_{\text{ML}}(\graph,\empX)$.
        If  $\hat{\mathcal{G}}\in \mathbb{G}^*$ and 
        $\hat{\graph}\in \argmin_{\graph\in\mathbb{G}^*} \dim(\graph)$,
        then 
        $\hat{\mathcal{G}}$ and $\mathcal{G}^*$ are algebraic equivalent, i.e.,
         $H(\hat{\mathcal{G}})=H(\mathcal{G}^*)$, in the large sample limit.
    \end{restatable}

    \begin{remark}
        \cref{thm:equivalence_equality_constraints} (partially inspired by \cite{ngscore})
        says that,
       if $\hat{\graph}$ can generate the observation $\empX$,
       and $\hat{\graph}$ has the smallest dimension among those graphs that can generate the observation,
       then $\hat{\graph}$ and $\graph^*$ are algebraically equivalent.
       In other words, we can enumerate all the graphs in the assumed graph space, and utilize \cref{thm:equivalence_equality_constraints} to 
         find a graph $\hat{\graph}$ that  is algebraically equivalent to the ground truth $\graph^*$.
        In the absence of latent variables, if $\hat{\graph}$ and $\graph^*$ are algebraically equivalent,
         $\hat{\graph}$ and $\graph^*$ belong to the same MEC,
        and thus we can make use of score 
        to identify the structure up to MEC. In this sense, \cref{thm:equivalence_equality_constraints} 
        takes the theoretical guarantee of score, e.g., in GES \citep{chickering2002optimal} 
        as a special case and generalizes it to the partially observed scenario.
        \vspace{-0.5em}
    \end{remark}

    Algebraic equivalence is a general sense of equivalence: if two graphs are algebraically equivalent, then we cannot differentiate them purely by observational data without any further  assumption.
    The reason lies in that, in the linear gaussian case, all the information from the distribution are just equality  and inequality constraints {(in \cref{def:HG})},
    and generally inequality constraints cannot be used; we have to assume inequality-constraint-faithfulness 
    to utilize inequality-constraints,
    but the set of parameters that results in violation of such faithfulness is not of Lebesgue measure zero.

    At this point,  we have characterized how { the} likelihood score can be used for
    structure identification in the partially observed scenario.
    Yet, there still exist two main challenges.
    First, in the partially observed scenario, without any graphical assumption, relating $\hat{\graph}$ to $\graph^*$ can be very challenging,
    as the algebraic equivalence class is very large.
    For example,
    suppose the ground truth graph is $\graph^*$ in \cref{fig:1}(a).
    Even though by making use of all the equality constraints, we still just arrive at 
    some elements of the algebraic equivalence class of $\graph^*$,
    e.g., $\hat{\graph}_1$ or $\hat{\graph}_2$ in \cref{fig:1}(b) and (c).
    In fact, without any further graphical assumption, the cardinality of the class is infinity, as we can always add one more latent variable to the structure
     without changing any constraints.
     Therefore, a general recipe may involve \textit{identifying suitable structural assumptions that allow algebraic equivalence to translate into more
     fine-grained notions of model equivalence, such as Markov
     equivalence}, as in \cref{sec:graphical_assumption_and_consistency}. 

    Second, \cref{thm:equivalence_equality_constraints} only implies a search procedure that requires the exact enumeration of all possible graphs in the assumed graph space.
    Yet,  such an exact search is impractical due to the computational overhead.
    To be specific, the number of possible graphs grows super-exponentially with the increase of the number of observed variables,
    { even when we rule out a lot of structures with latent variables that cannot be identified (e.g., $\node{X}_1\rightarrow\node{L}_1\rightarrow\node{X}_2$ where we can never know whether $\node{L}_1$ exists or not).
    More specifically, if we consider the GNFM (which will be detailed in the next section with definition in \cref{definition:gnfm})
    with 10 observed variables,}
    the number of all possible graphs 
    that satisfy \cref{definition:gnfm} is more than $3\times10^8$
    and this number grows to $2\times10^{13}$ when the number of observed variables increase only by 1.
    Therefore, it is crucial to design \textit{an efficient way to search over the graph space for the identification of the optimal structure}, as  in \cref{sec:method}.

\vspace{-1em}
\subsection{Graphical Assumption and Global Consistency}
\vspace{-0.5em}
\label{sec:graphical_assumption_and_consistency}



In this section, we propose {the} generalized N factor model,
a graphical condition under which 
algebraic equivalence can be translated into Markov equivalence,
defined as follows.

\begin{definition}[Generalized N Factor Model]\label{definition:gnfm}
    DAG $\graph$ satisfies the definition of generalized N factor model if observed variables are the effects of latent variables
    and there exists a partition of all latent variables $\set{L}_\graph$ such that 
    for each element in the partition, $\set{L}_{p}$,
    (i) 
        there exist at least $|\set{L}_{p}|*2$ observed variables $\set{X}_{p}$
    such that for all $\node{X}\in\set{X}_{p}, \parents(\node{X})=\set{L}_{p}$,
    (ii)
    if { a variable $\node{V}\in\set{V}_\graph$} causes or is caused by another variable in $\set{L}_p$, then
    $\node{V}$ also causes or is caused by all the variables in $\set{L}_p$, respectively,
    and
    (iii) elements in $\set{L}_{p}$ are mutually nonadjacent.
    \vspace{-0.5em}
\end{definition}
    
For brevity, in the rest of the paper, we use $\mathbb{G}_{\text{GNFM}}$ 
to denote the set of all graphs that satisfy the definition of generalized N factor model.
For a better understanding of the generalized N factor model,
we provide an example to elaborate each requirement in \cref{definition:gnfm}.

\begin{wrapfigure}{r}{0.4\textwidth}
  \centering
  \vspace{-1.5em}
  \includegraphics[width=0.38\textwidth]{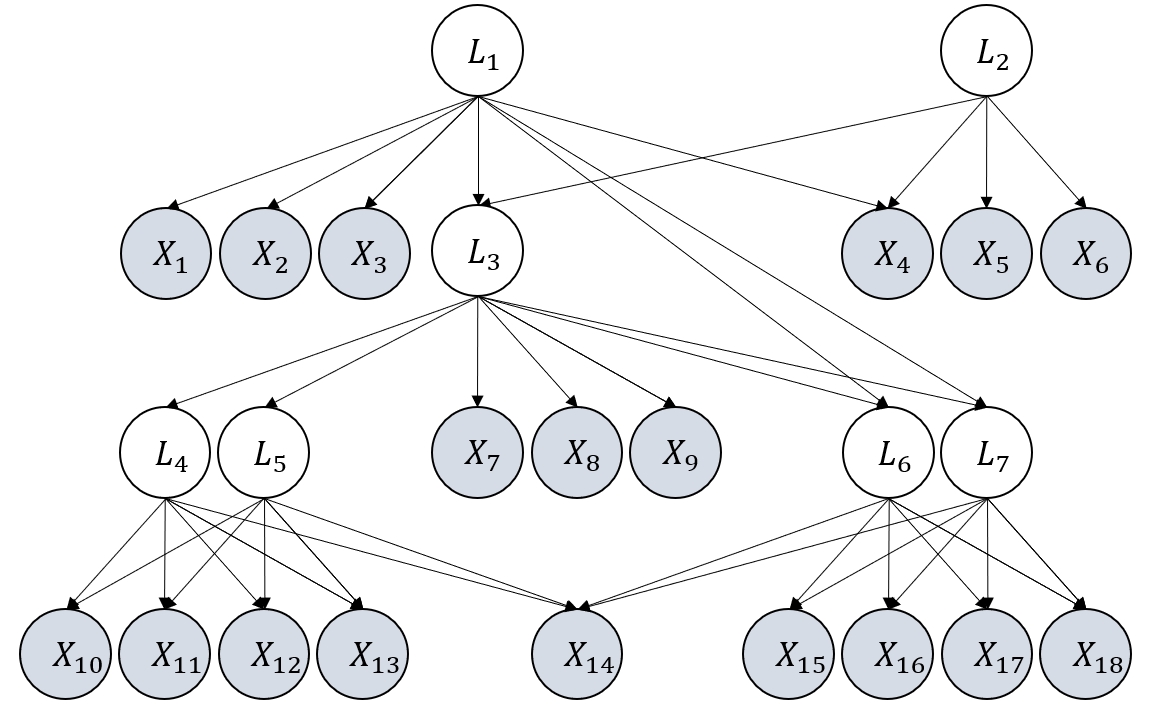}
  \caption{An illustrative example of the graph that satisfies generalized N factor model in \cref{definition:gnfm}.
    }
  \label{fig:gnfm_illustration}
  \vspace{-1.5em}
\end{wrapfigure}

Fig.~\ref{fig:gnfm_illustration} shows an illustrative graph that satisfies Def.~\ref{definition:gnfm}.
Specifically, all the observed variables in \cref{fig:gnfm_illustration}
are leaf nodes and 
 caused by only latent variables. Further, 
all latent variables 
can be partitioned into groups $\{\node{L}_1\}, \{\node{L}_2\},\{\node{L}_3\},\{\node{L}_4,\node{L}_5\},\{\node{L}_6,\node{L}_7\}$
and thus the rest  requirements (i), (ii), (iii) are also satisfied.
For example, for $\set{L}_p=\{\node{L}_6,\node{L}_7\}$,
there exist $\set{X}_p=\{\node{X}_{15},\node{X}_{16},\node{X}_{17},\node{X}_{18}\}$
such that their parents are $\{\node{L}_6,\node{L}_7\}$
and $|\set{X}_p|\geq|\set{L}_p|\times2$.
These groups 
have the properties that
the relation within a group can not be identified, 
and the relation between groups applies to every element in the group.
Thus,  requirements in Def. \ref{definition:gnfm} 
are satisfied in \cref{fig:gnfm_illustration}.

Notably, 
the graphical condition required for generalized N factor model is rather weak. It takes the prevalent one factor model assumption \citep{silva2003learning} (defined in \cref{definition:silva_graphical_criterion} and illustrated in \cref{sec:appx_one_factor_model} for a comparison) as a special case,
and further allows latent variables to share  observed variables as children. 
In GNFM, we want to make very weak assumption about how latent variables are related - they can be partitioned into groups, and the relation among each  group of latent variables can be very flexible.
At this premise, the key requirement of  $2|\mathbf{L}_{p}|$ observed children of $\mathbf{L}_{p}$  in GFNM is minimal. 
On the other hand, if we make a stronger assumption about how latent variables are related to each others than what has been made in GNF{M},
the requirement of having at least twice number of observed children can be relaxed. 

{Here we take 
$\{\node{L}_3\}$ 
in Fig. \ref{fig:gnfm_illustration} as $\set{L}_{p}$ 
to show why we need $2|\set{L}_{p}|$ observed children of $\set{L}_{p}$ . If we want to determine whether two latent groups $\{\node{L}_1\}$ and $\{\node{L}_4,\node{L}_5\}$ are  adjacent,
 we need to check whether $\{\node{L}_1\}$ and $\{\node{L}_4,\node{L}_5\}$ can be d-separated by $\{\node{L}_3\}$,
 which can be done by checking whether there exists a pure observed child of $\{\node{L}_1\}$ (e.g., $\node{X}_1$) and another pure observed child of $\{\node{L}_4,\node{L}_5\}$  (e.g., $\node{X}_{10}$) can be 
  d-separated by $\{\node{L}_3\}$. 
 As $\{\node{L}_3\}$ cannot be observed,
 we have to rely on t-separation \citep{sullivant2010trek} to check whether this d-separation holds, 
 as t-separation allows us to use the observed children of $\{\node{L}_3\}$ as surrogates.
  This can be translated into checking whether there exists two distinct groups $\set{X}_a$ and $\set{X}_b$ as the observed pure children of 
  $\{\node{L}_3\}$ such that, 
  $|\set{X}_b|=|\set{X}_b|=|\{\node{L}_3\}|$, and   
  $\{\node{X}_1\}$ and $\{\node{X}_{10}\}$
  can be t-separated by $(\set{X}_a,\set{X}_b)$.
  Thus, for each $\set{L}_{p}$ we need $2|\set{L}_{p}|$ observed children of it.
 }


The reason why GNFM  chooses this specific trade-off point is that we believe it is more practically meaningful. In real-life problems, 
the way latent variables are related could be complicated
and we do not want to  rule out the possibility of certain structural patterns in advance. 
At the same time, if the number of observed children is insufficient, we can still gather more relevant observations or measurements of the underlying system. 
\looseness=-1

Next, we show in \cref{thm:identify_gnfm} that
for GNFM, the notion of algebraic equivalence
leads to Markov equivalence,
and thus  by using score we can identify the structure up to MEC,
as in  \cref{corollary:global_consistency}.

\begin{restatable}[Identifiability of Generalized N Factor Models by Equality Constraint up to MEC]{theorem}{TheoremIdentifiabilityGNFM}\label{thm:identify_gnfm}
For  $\graph_1,\graph_2 \in \mathbb{G}_{\text{GNFM}}$,
if they are algebraically equivalent, i.e., $H(\graph_1)=H(\graph_2)$, then $\graph_1$ and $\graph_2$ belong to the same MEC (same skeleton and v-structures over all variables).
\vspace{-0em}
\end{restatable}




\begin{restatable}[Global Consistency by Score for Generalized N Factor Models]{corollary}{CorollaryGlobalConsistency}\label{corollary:global_consistency}
    Suppose a model follows \cref{definition:polcm} with $\graph^*\in \mathbb{G}_{\text{GNFM}}$
        and distribution $\Sigma^{*}_\set{X}$ satisfies \cref{assumption:generalized_faithfulness}.
        Given observation $\empX$
        and let $\mathbb{G}^* = \arg\max_{\graph\in\mathbb{G}_{\text{GNFM}}}
        \text{score}_{\text{ML}}(\graph,\empX)$.
        If  $\hat{\mathcal{G}}\in \mathbb{G}^*$ and 
        $\hat{\graph}\in \argmin_{\graph\in\mathbb{G}^*} \dim(\graph)$,
        then 
        $\hat{\mathcal{G}}$ and $\mathcal{G}^*$ are Markov equivalent in the large sample limit.
        \vspace{-0.5em}
\end{restatable}


\cref{corollary:global_consistency} {establishes} the global consistency
of using likelihood score for structure identification.
Yet, it still requires
 impractical exact enumeration.
Furthermore, it is 
also unclear how to calculate the exact dimension 
for each graph with latent variables.
Fortunately, an efficient greedy search can be designed 
to identify the optimal structure
without capturing the dimension, as  shown in what follows.

\vspace{-1em}
\section{Score-based Greedy Search for Partially Observed Linear Causal Models}
\vspace{-0.5em}
\label{sec:method}
%
We begin with the general design for greedy search in \cref{sec:general_design},
specify the design in LGES for generalized N factor model in \cref{sec:design_gnfm},
and establish the asymptotic correctness  in \cref{sec:correctness_of_lges}.

\vspace{-0.5em}
\subsection{General Design for Greedy Search}
\vspace{-0.5em}
\label{sec:general_design}

\begin{wrapfigure}{r}{0.4\textwidth}
\vspace{-2em}
\setlength{\belowcaptionskip}{0mm}
  \vspace{-0.5em}
    \subfloat[$\state_{\text{init}}$ by Def. \ref{def:init_state},
    where latents are mutually fully connected by undirected edges and all latents cause all observed.]
    {
    \centering
    \includegraphics[width=0.36\textwidth]{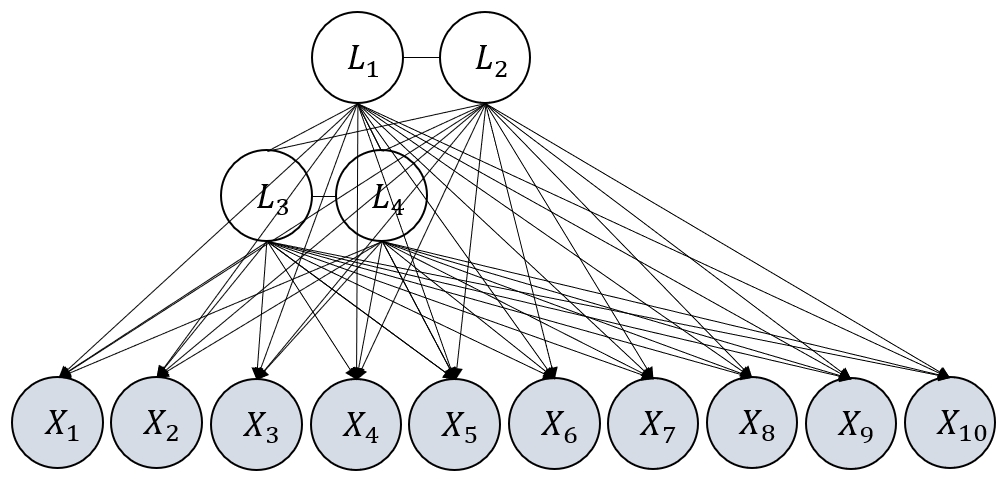}
    }
    \hspace{0.6em}
    \subfloat[$\state_{\text{phase1}}$, the output of Alg. \ref{alg:lges_phase1},
    where the number of latent and  the edges from latent to observed variables are determined.]
    {
    \centering
    \includegraphics[width=0.36\textwidth]{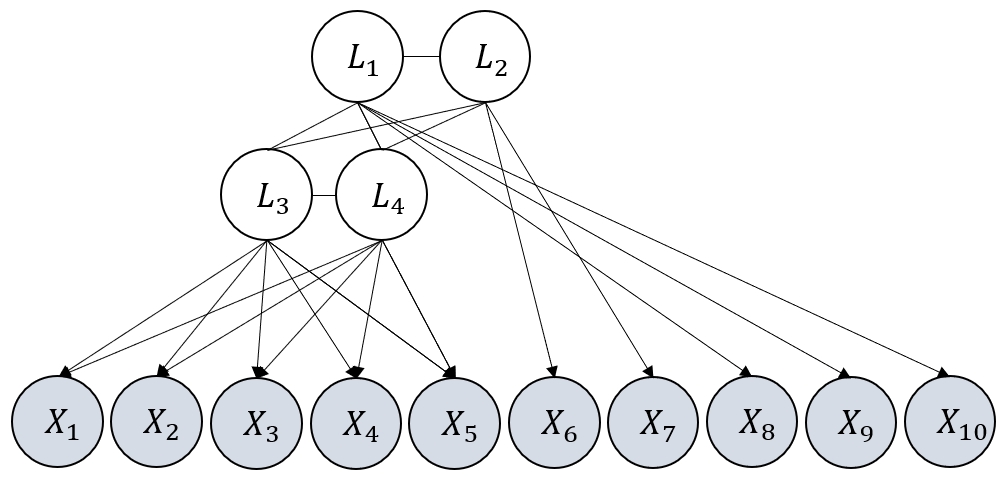}
    }
    \hspace{0.6em}
    \subfloat[$\state_{\text{final}}$, the output of Alg. \ref{alg:lges_phase2},
     represents the MEC of the ground truth $\graph^*$ (here the MEC only contains $\graph^*$).
     ]
    {
    \centering
    \includegraphics[width=0.36\textwidth]{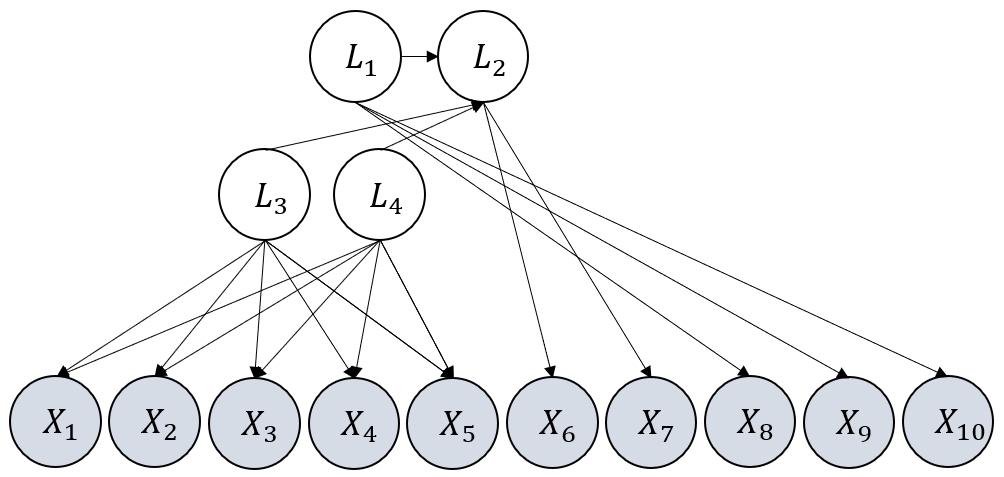}
    }
    \centering
    \vspace{-0em}
\caption{
 LGES illustration, where (a) is the initial state of Alg. \ref{alg:lges_phase1},
(b) is the output of  Alg. \ref{alg:lges_phase1}, and (c) is the final output.
 }
  \label{fig:process}
  \vspace{-5em}
\end{wrapfigure}

Our overall objective  is to
find a graph $\hat{\graph}$ 
such that it can generate $\empX$ equally well as the ground truth $\graph^*$,
while having a dimension that is as small as possible (as in \cref{thm:equivalence_equality_constraints} and \cref{corollary:global_consistency}).
To efficiently search over the graph space,
we follow the traditional wisdom GES  {\citep{chickering2002optimal}} and define three key elements for  a greedy search.
\vspace{-1mm}
\begin{itemize}[leftmargin=*, itemsep=-1pt]
\item A set of states.
\item A representation scheme for the states.
\item A set of operators.
\end{itemize}
\vspace{-0.5em}

A state represents a solution to the search problem
and we use $\state$ to represent state and $\setstate$ to represent a set of states.
The representation scheme defines an efficient way to represent the states.
As the structure can only be identified up to MEC,
{
we follow GES such that each state is a MEC,
which corresponds to an unique} Completed Partially Directed Graph (CPDAG  {in \cref{def:cpdag}}). Thus we also use $\state$ to refer to a CPDAG
and $\state(\graph)$ to transform a DAG $\graph$ into a CPDAG.
Finally, the set of operators is to transform one state to another state,
 in order to traverse the whole graph space systematically and efficiently.


For any two graphs  belong to the same state (i.e., the same MEC),
they share the same dimension and maximum likelihood score (Prop. 1 in \citep{ngscore}).
Thus, we also define $\text{score}_{\text{ML}}$ for a CPDAG $\state$,
as $\text{score}_{\text{ML}}(\state,\empX)=\text{score}_{\text{ML}}(\graph,\empX)$, for all $\graph\in\text{MEC}(\state)$.

Furthermore, the initial state $\state_{\text{init}}$ is often designed as a state that is a super graph of the ground truth, and 
 $\state_{\text{init}}$ can generate the observation equally well as the ground truth.
For example, the phase 1 of GES is to find such an initial state and thus the phase 2 of GES can focus on the delete operation.
Next, we introduce the detailed design of our novel method, LGES.

\vspace{-0em}
\subsection{Algorithm: Latent variable Greedy Equivalence Search (LGES)}
\vspace{-0.5em}
\label{sec:design_gnfm}

In this section we discuss the detailed design for Latent variable Greedy Equivalence Search (LGES)
for structure identification of generalized N factor models.
We begin with the initial state $\state_{\text{init}}$.

\begin{definition}[Initial State for Generalized N Factor Model]
    \label{def:init_state}
    Given $\set{X}$, $\state_{\text{init}}(\set{X})$ outputs a CPDAG such that
    it contains observed variables $\set{X}$ and $\lfloor \frac{|\set{X}|}{2}\rfloor$ latent variables,
     all latent variables are fully connected with undirected edge, and all latent variables cause all observed variables.
\end{definition}

In \cref{def:init_state}
it implicitly requires 
that the number of observed variables is at least twice the number of latent ones. We note that this latent-to-observed ratio is a property of the graphical assumption of GNFM in \cref{definition:gnfm}, and  thus invariant to the design of a method.
An example  of $\state_{\text{init}}(\set{X})$ can be found in \cref{fig:process},
where (c) is the ground truth $\graph^*$ and (a) is the initial state $\state_{\text{init}}(\set{X})$ by \cref{def:init_state}.
The reason why we design $\state_{\text{init}}$ as such lies in the properties of $\state_{\text{init}}$ formalized in  \cref{lemma:init_can_generate}.

\begin{restatable}[Properties of Initial State]{lemma}{LemmaInitCanGenerate}\label{lemma:init_can_generate}
    Suppose a model follows \cref{definition:polcm} with $\graph^*\in \mathbb{G}_{\text{GNFM}}$ and we are given observation $\empX$.
    Then $\state_{\text{init}}$ is a supergraph of $\state(\graph^*)$
    and $\state_{\text{init}}$ can generate the observed distribution,
    i.e., $\text{score}_{\text{ML}}(\state_{\text{init}},\empX) = \max_{\graph\in\mathbb{G}_{\text{GNFM}}}
    \text{score}_{\text{ML}}(\graph,\empX)$ in the large sample limit.
    \vspace{-0.5em}
\end{restatable}

The core spirit of LGES is that we begin with a state that can generate the observation. Each time we delete some edges and see whether the new CPDAG can still generate the observation.
If so, we keep it as the new state; otherwise we try a different deletion.
%
As an edge deletion leads to smaller or equal dimension, the process is in essence finding the CPDAG with smallest dimension while keeping the ability of generating the observation.
Next we discuss the operators that connect between states.

\begin{definition}[Delete Operator $\operator_{\set{L}\set{X}}$]
    \label{def:operatorlx}
    \vspace{-0em}
    $\operator_{\set{L}\set{X}}(\state,\set{L},\set{X})$ returns a CPDAG that is the same as $\state$ except that 
    all edges from $\set{L}$ to $\set{X}$ are deleted.
\end{definition}

\begin{definition}[Delete Operator $\operator_{\set{L}\set{L}}$]
    \label{def:operatorll}
    $\operator_{\set{L}\set{L}}(\state,\set{L}_1,\set{L}_2,\set{H})$ returns a CPDAG that is the same as $\state$ except that 
    (i) all edges between $\set{L}_1$ and $\set{L}_2$ are deleted,
    (ii) for each $\node{H}\in\set{H}$,
    directing the previously undirected edge between $\set{L}_1$ and $\node{H}$ as $\set{L}_1\rightarrow\node{H}$
    and directing the previously undirected
    edge between $\set{L}_2$ and $\node{H}$ as $\set{L}_2\rightarrow\node{H}$.
    \vspace{-0.5em}
\end{definition}


$\operator_{\set{L}\set{X}}$ is designed to delete edges from latent variables to observed variables
while 
$\operator_{\set{L}\set{L}}$ (partially inspired by \citet{chickering2002optimal}) is designed to delete relations among  two groups of latent variables  $\set{L}_1$ and $\set{L}_2$.
%
Our LGES is built upon these two operators and includes two phases.
Phase 1 (summarized in Alg \ref{alg:lges_phase1}) is to recover the structure between latent variables and observed variables.
Roughly speaking, each time the algorithm chooses some latent variables and some observed variables from the current state, 
 deletes all the edges between them to get a neighbouring state, and check whether this neighbouring state can generate the 
 observation $\empX$ (up to a certain tolerance level $\delta$).
 If so, this neighboring state is taken as the new state; otherwise the algorithm looks for a new combination for edge deletion.
An example of the output of phase 1 is in Fig \ref{fig:process} (b),
 where the structure between latent and observed are expected to be the same as the ground truth (also formalized in \cref{lemma:correctphase1}).

Phase 2 (summarized in Alg \ref{alg:lges_phase2}) of LGES, which is partially inspired by \citet{chickering2002optimal},
aims to identify the structure among latent variables. 
Roughly, each time the algorithm chooses two groups of latent variables, 
deletes all the edges between them to get a neighbouring state;
if this state can generate the 
observation (up to $\delta$),
then takes it as the new state; otherwise looks for another two groups.
$\delta$ controls the sparsity level in the finite sample case:
with bigger $\delta$, deletions are more likely to be kept (ablation study in \cref{sec:synthetic_data}).
An illustration of the output of phase 2 is in Fig \ref{fig:process} (c),
which is expected to be the same as the ground truth (which is formalized in \cref{theorem:correctphase2}).

As implied by \cref{thm:equivalence_equality_constraints},
our objective is to search over the graph space to find a graph $\hat{\mathcal{G}}$ such that (i) $\hat{\mathcal{G}}$
 has the best likelihood and  (ii) at the premise of ensuring (i), the dimension of $\hat{\mathcal{G}}$ should be as small as possible. 
 Below are the key designs of LGES to ensure this goal. 
\vspace{-0.5em}
 
 \begin{itemize}[leftmargin=*, itemsep=-1pt]
\item We start with a graph that is guaranteed to be a super graph of the ground truth.

\item At each step, we try to delete some edges. Only when the likelihood after the deletion is still the best do we keep the deletion. This guarantees that (i) always holds.

\item  An edge deletion operation will either keep the dimension the same or decrease the dimension. Thus, throughout our search process, the dimension will decrease monotonically.

\item The design of deletion operators ensure at each step, the current graph  is always a super-graph of the ground truth, through out the whole process.  This is because if the post-deletion graph is not a super-graph of the ground truth, additional equality constraints will be introduced such that the  post-deletion graph cannot reach the best likelihood and thus this deletion will not be kept.

 \item By the end of the process, LGES will arrive at the ground truth. This can be proved by contradiction in a rough sense as follows (detailed in proof). 
Suppose the final state $\mathcal{G}'$ is not optimal. As it must be a super graph of  $\mathcal{G}^*$, there must exist a sequence of deletion to transform  
$\mathcal{G}'$ into $\mathcal{G}^*$. Suppose the first deletion in the sequence leads to $\mathcal{G}''$. As $\mathcal{G}''$ is also a super graph of
$\mathcal{G}^*$, the score of $\mathcal{G}''$ is also the best, and thus the algorithm would not terminate at $\mathcal{G}'$, yielding a contradiction.
\end{itemize}

\vspace{-1em}
\subsection{Asymptotic Correctness of LGES}%
\vspace{-0.5em}
\label{sec:correctness_of_lges}
Here we establish the asymptotic correctness of LGES.
Specifically,
if the ground truth satisfies the generalized N factor model,
LGES can asymptotically produce the correct Markov equivalence class over both observed and
latent variables, as in \cref{lemma:correctphase1} and \cref{theorem:correctphase2} (all proofs in Appendix).

\begin{restatable}[Correctness of Phase 1 of LGES]{lemma}{LemmaCorrectnessPhaseone}\label{lemma:correctphase1}
    Suppose a model follows \cref{definition:polcm} with $\graph^*\in \mathbb{G}_{\text{GNFM}}$ and we are given observation $\empX$.
    In the large sample limit the output $\state_{\text{phase1}}$ of \cref{alg:lges_phase1} is a CPDAG such that the number of latent variables in
    $\state_{\text{phase1}}$ is the same as that of $\graph^*$
    and the edges from $\set{X}$ to $\set{L}$ in $\state_{\text{phase1}}$
    is the same as that of $\graph^*$,
    up to permutation of latent variables.
\end{restatable}

\begin{restatable}[Correctness of LGES]{theorem}{TheoremCorrectnessPhaseTwo}\label{theorem:correctphase2}
    Suppose a model follows \cref{definition:polcm} with $\graph^*\in \mathbb{G}_{\text{GNFM}}$ and we are given observation $\empX$.
    In the large sample limit the output $\state_{\text{final}}$ of \cref{alg:lges_phase1,alg:lges_phase2} is a CPDAG 
    that represent the MEC of $\graph^*$,
    up to permutation of latent variables.
\end{restatable}

\begin{center}
  \begin{table*}[tb]
  \vspace{-0mm}
  \vspace{-1em}
    \caption{F1 score and SHD (mean(standard error)) of each compared method (LGES, FOFC, GIN, and RLCD) across different sample sizes.
      $\uparrow$  means the bigger the better while $\downarrow$ the smaller the better.}
     \vspace{-0.5em}
     \label{tab:result}
    \footnotesize
    \center 
  \begin{center}
  \resizebox{\linewidth}{!}{
  \begin{tabular}{|c|c|c|c|c|c|c|c|c|c|}
    \hline  \multicolumn{1}{|c|}{} &\multicolumn{4}{c|}{{F1 score for skeleton $\uparrow$} }&\multicolumn{4}{c|}{{SHD for MEC $\downarrow$} }\\
    \hline 
    \multicolumn{1}{|c|}{Sample size}  & \textbf{LGES} & FOFC & GIN & RLCD & \textbf{LGES} & FOFC & GIN & RLCD\\
    \hline 
    100 & \textbf{0.60} (0.01) & 0.55 (0.02) & 0.27 (0.01)&0.33 (0.01) &\textbf{20.87} (0.72) & 20.8 (0.34) &26.76 (0.36) &35.84 (3.01)\\
    \hline 
    200 & \textbf{0.72} (0.02) &0.56 (0.02) &0.36 (0.01) &0.49 (0.02) &\textbf{14.03} (0.85) & 18.0 (0.29) &24.14 (0.45) &28.58 (2.59)\\
    \hline 
    500 & \textbf{0.79} (0.02) &0.58 (0.02) &0.42 (0.01) &0.69 (0.02) &\textbf{10.02} (0.69) & 16.5 (0.31) &22.46 (0.39) &14.68 (0.96)\\
    \hline 
    1000& \textbf{0.82} (0.02)  &0.61 (0.03)  & 0.47 (0.01) &0.76 (0.02)  & \textbf{8.80} (0.70) &  16.1 (0.29)& 20.88 (0.54) &11.24 (0.86)\\
    \hline 
  \end{tabular}
  }
  \end{center}
  \vspace{-0.5em}
  \end{table*}
  \end{center}

\begin{center}
  \begin{table*}[tb]
  \vspace{-0mm}
    \caption{Under model misspecification,
    F1 score and SHD  (mean(standard error)) 
 of each compared method across different sample sizes.
      $\uparrow$  means the bigger the better while $\downarrow$ the smaller the better.}
     \vspace{-0.5em}
     \label{tab:misspecification}
    \footnotesize
    \center 
  \begin{center}
  \resizebox{\linewidth}{!}{
  \begin{tabular}{|c|c|c|c|c|c|c|c|c|c|}
    \hline \textbf{Non-gaussian} &\multicolumn{4}{c|}{{F1 score for skeleton $\uparrow$} }&\multicolumn{4}{c|}{{SHD for MEC $\downarrow$} }\\
    \hline 
    \multicolumn{1}{|c|}{Sample size}  & \textbf{LGES} & FOFC & GIN & RLCD & \textbf{LGES} & FOFC & GIN & RLCD\\
        \hline 
   {100}& {\textbf{0.57} (0.03)} & {0.40 (0.03)} & {0.27 (0.01)} & {0.35 (0.01)} & {\textbf{21.70} (1.67)} &  {22.72 (0.34)} &{26.76 (0.36)} &{39.31 (7.37)}\\
        \hline 
    { 200} & {\textbf{0.74} (0.04)} & {0.42 (0.03)} & {0.36 (0.01)} & {0.46 (0.03)} & {\textbf{12.67} (1.97)} & {18.10 (0.50)} & {24.14 (0.45)} &{35.12 (6.52)}\\
    \hline 
    500 & \textbf{0.78} (0.02) &0.40 (0.03) &0.42 (0.01) &0.68 (0.01) &\textbf{10.72} (0.93) & 17.66 (0.49) &22.46 (0.39) &14.81 (0.63)\\
    \hline 
    1000& \textbf{0.79} (0.02)  &0.39 (0.03)  & 0.47 (0.01) &0.75 (0.02)  & \textbf{10.12} (0.82) &  17.82 (0.44)& 20.88 (0.54) &12.45 (0.77)\\
    \hline 
    \hline 
    \textbf{Non-linear} &\multicolumn{4}{c|}{{F1 score for skeleton $\uparrow$} }&\multicolumn{4}{c|}{{SHD for MEC $\downarrow$} }\\
    \hline 
    \multicolumn{1}{|c|}{Sample size}  & \textbf{LGES} & FOFC & GIN & RLCD & \textbf{LGES} & FOFC & GIN & RLCD\\
    \hline 
   {100} & {\textbf{0.58} (0.02)} & {0.35 (0.03) }& {0.32 (0.01)} & {0.36 (0.02)} &{\textbf{19.21} (1.29)} & {21.48 (0.44)} &{29.28 (0.11)} &{37.01 (4.77)}\\
            \hline 
    {200} & {\textbf{0.65} (0.01)} &{0.33 (0.03)} & {0.38 (0.01)} & {  0.45(0.02)} & { \textbf{18.22} (1.11)} & {21.80 (0.41)} &{27.28 (0.28)} & {31.73 (4.33)}\\
    \hline 
    500 & \textbf{0.68} (0.02) &0.26 (0.04) &0.50 (0.02) &0.60 (0.02) &\textbf{16.6} (1.24) & 19.98 (0.46) &23.38 (0.64) &19.89 (1.08)\\
    \hline 
    1000& \textbf{0.71} (0.02)  &0.23 (0.04)  & 0.52 (0.02) &0.65 (0.01)  & \textbf{15.0} (0.95) &  20.10 (0.54)& 23.32 (0.73) &17.56 (0.77)\\
    \hline 
  \end{tabular}
  }
  \end{center}
  \vspace{-1em}
  \end{table*}
  \end{center}

\vspace{-1.5em}
\section{Experiments}
\subsection{Synthetic Setting and Evaluation Metric}
\vspace{-0.5em}
%
In our  synthetic experiments,
we validate LGES by comparing it with 
multiple latent variable causal discovery methods, including
FOFC~\citep{kummerfeld2016causal}, GIN~\citep{xie2020generalized},
and RLCD~\citep{dong2023versatile}.
The ground truth model follows Def. \ref{definition:polcm}
where  each edge coefficient $f_{ji}$ is randomly sampled uniformly from $[-5, 5]$
and the variance for each noise term uniformly from $[0.1,1]$.
As GIN assumes non-gaussianity, we use uniform distribution for the noise when testing GIN.
 We consider 20 randomly generated structures
that satisfy~Def. \ref{definition:gnfm} as the ground truth structure (examples in Fig. \ref{fig:illustration_of_graphs} in Appendix). 
 On average, each ground truth graph contains 20 variables and 5 of them are latent.
 { For the output of each method, we first transform the output into a CPDAG and then compare it with the ground truth CPDAG by calculating} F1 score over the skeleton and SHD over the MEC as  evaluation metrics,
 and consider four different sample sizes: 100, 200, 500, 1000.
 We adopt 10 random seeds to generate the ground truth causal model and observational data, and report the mean and standard error.
 \looseness=-1

\begin{table}[t]
\vspace{-1em}
\begin{minipage}[c]{0.5\linewidth}
\centering
    \caption{Model fitness scores on { r}eal-life datasets to validate LGES ($\uparrow$  the bigger the better).}
    \label{tab:fitness}
    \vspace{-1em}
    \resizebox{\linewidth}{!}{
    \begin{tabular}{|c|c|c|c|c|}
        \hline  \multicolumn{2}{|c|}{Scenarios} &\multicolumn{3}{c|}{{Model Fitness Scores}}\\
        \hline 
         {Dataset} &{Structures} & RMSEA \textbf{$\downarrow$}  & CFI \textbf{$\uparrow$}& TLI\textbf{$\uparrow$} \\
         \hline 
        \multirow{2}{*}{{Big Five}} & \tiny{\textbf{\cref{fig:bigfive} by LGES}}  & \textbf{0.054} & \textbf{0.874} &\textbf{0.855}\\
        \cline{2-5}
         & \tiny{ By \citet{Goldberg1993}}
          & 0.072& 0.767 & 0.746\\
          \cline{2-5}
            \hline
        \multirow{2}{*}{{T Burnout}} & \tiny{\textbf{\cref{fig:burnout} by LGES}} & \textbf{0.067} & \textbf{0.876} &\textbf{0.865}\\
        \cline{2-5}
         & \tiny{ By \citet{Byrne1994b}}  & 0.096& 0.753&0.727\\
        \cline{2-5}
             & \tiny{By \citet{Byrne2010}}  & 0.072& 0.861&0.847\\
        \hline
        \multirow{2}{*}{{M-tasking}} & \tiny{\textbf{\cref{fig:multitasking} by LGES}}  & \textbf{0.068} & \textbf{0.977} &\textbf{0.965}\\
        \cline{2-5}
         & \tiny{By \citet{Himi2019}} &0.087 &0.962 &0.943\\
        \hline
      \end{tabular}
      }
\end{minipage}\hfill
\hspace{0.5em}
\begin{minipage}[c]{0.48\linewidth}
  \centering
  \vspace{1em}
\includegraphics[width=\linewidth]{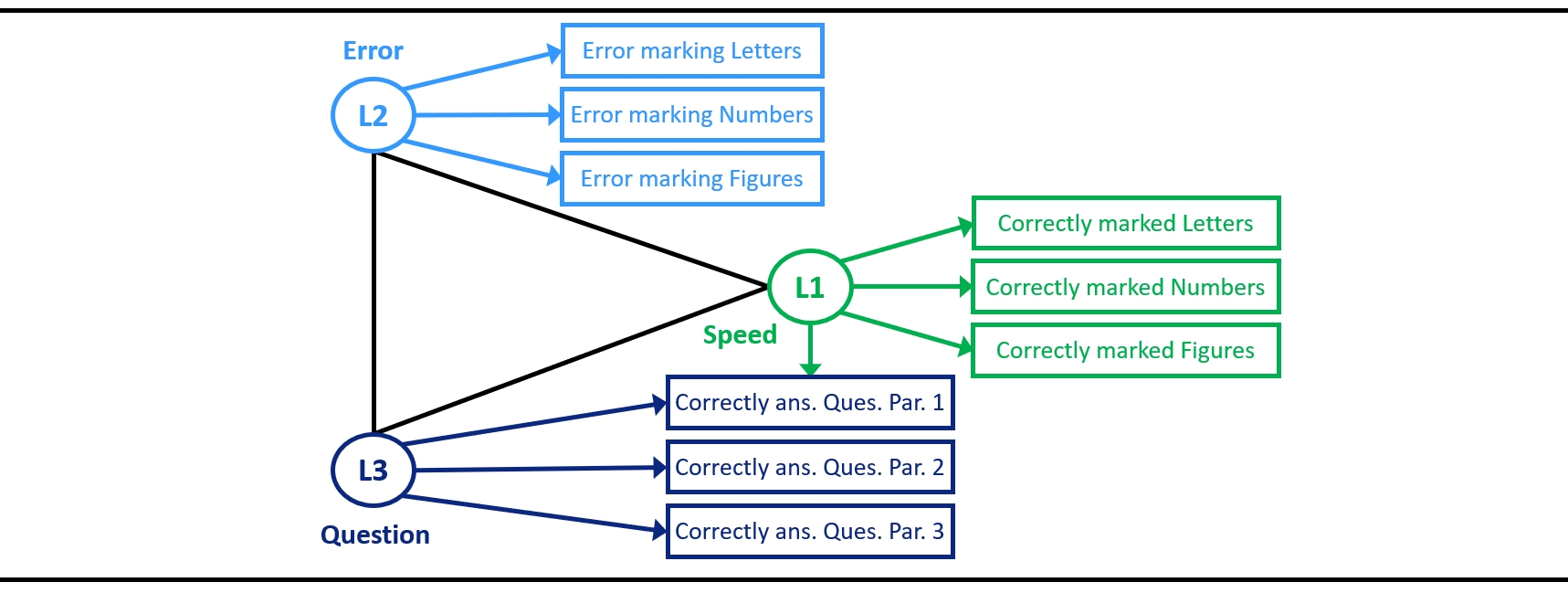}
\vspace{-2em}
 \captionof{figure}{\small Causal structure (CPDAG) recovered by LGES on Multi-tasking behavior dataset.}
    \label{fig:multitasking}
\end{minipage}
\end{table}

 \vspace{-0.5em}
\subsection{Performance on Synthetic Data and 
Ablation Study on $\delta$}
\vspace{-0.5em}
\label{sec:synthetic_data}
The F1 scores of skeletons and SHDs of MECs are reported  in Table~\ref{tab:result}.
As shown, the proposed LGES achieves the best F1 score (bigger better) and SHD (smaller better) performance compared to all baselines.
For example, with sample size 1k, LGES achieves a F1 score of 0.82 and a SHD of 8.8, where the runner-up achieves 0.76 and 11.24 respectively.
Another key observation is that our score-based method can still work well with a very small sample size,
while  constraint-based methods may not.
%
E.g., with only 100 datapoints, LGES still achieves a F1 score of 0.60, while constraint-based method GIN (CI-test-based) and RLCD (rank-test-based)
only achieves 0.27 and 0.33 respectively.
The reason why constraint-based methods do not work well with a small sample size may lie in that
with small sample size the power of the test is limited and and the null distribution may be very different from the asymptotic case,
both of which aggravate the issue of error { propagation}.
On the contrary, our score-based method may not suffer from these issues.
\looseness=-1

 Similar to  the hyper-parameter $\lambda$ in GES, in practice we can tune the value of  $\delta$ to control the sparsity level of the result (the bigger $\delta$ the sparser). 
 In our experiments, $\delta$ is set as $\delta=0.25\times\frac{\log(N)}{N}$, where $N$ is the sample size. This design follows the spirit of BIC score in GES such that $\delta\rightarrow0$ when $N\rightarrow\infty$. 
    The ablation study to analyze the sensitivity to $\delta$ can be found in Table~\ref{tab:on_delta},
    where LGES is not very sensitive to small change of $\delta$. 
    A more detailed discussion about $\delta$ can be found in \cref{sec:appx_implementation}.
    \looseness=-1

\begin{table}[t]
\vspace{-1em}
\begin{minipage}[c]{0.48\linewidth}
\centering
\includegraphics[width=\linewidth]{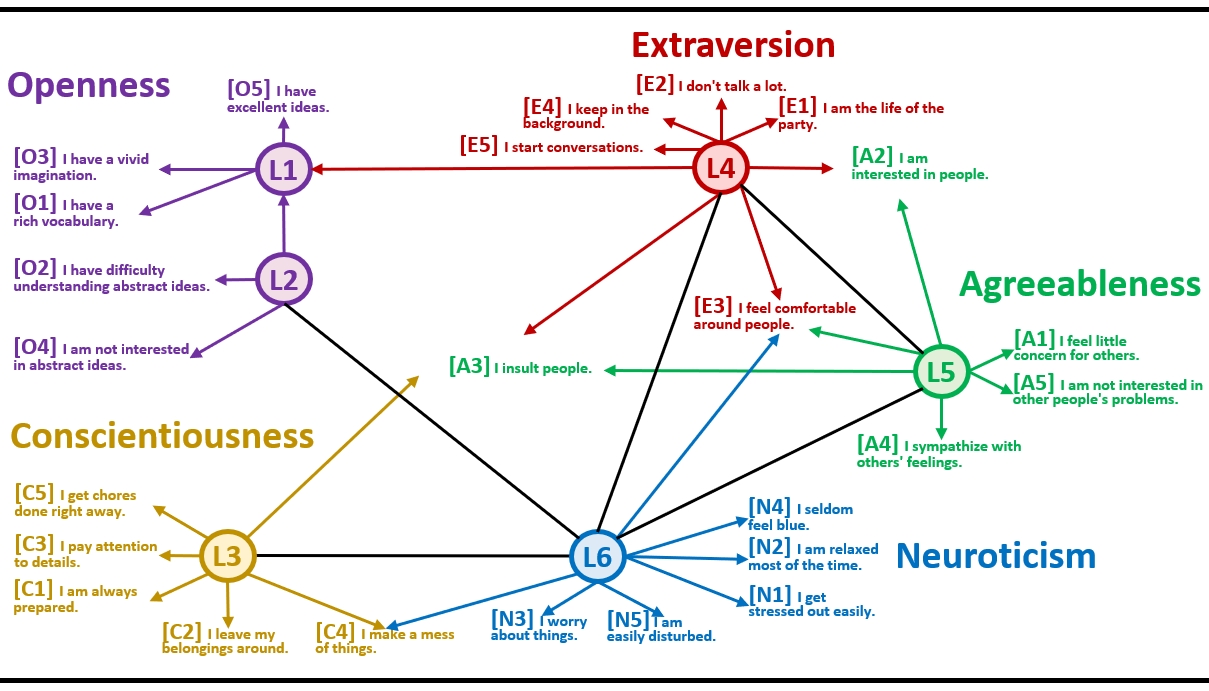}
\vspace{-2em}
\captionof{figure}{\small Causal structure (CPDAG) recovered by LGES on Big Five personality dataset.
    }
\label{fig:bigfive}
\end{minipage}
\hspace{0.5em}
\begin{minipage}[c]{0.4875\linewidth}
\centering
\vspace{0em}
\includegraphics[width=\linewidth]{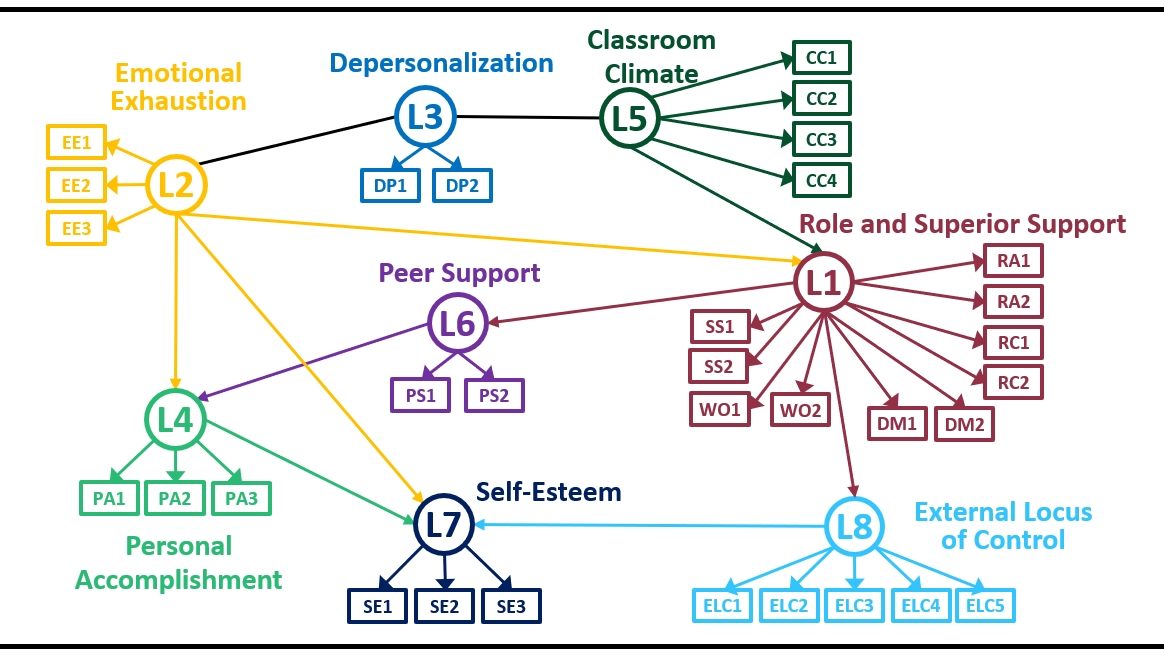}
\vspace{-2em}
 \captionof{figure}{\small Causal structure (CPDAG) recovered by LGES on Teacher burnout dataset. }
 \label{fig:burnout}
\end{minipage}\hfill
\vspace{-0.5em}
\end{table}

\vspace{-0.5em}
\subsection{Misspecification Behavior}
\vspace{-0.5em}
We 
investigate the performance of LGES 
 under model  misspecification:
violation of normality and linearity.
For the setting of violation of normality, we use uniform noise terms for the underlying model,
and report the result in \cref{tab:misspecification}.
Specifically, when the normality is violated, LGES  still performs the best compared to baselines
and the result is almost the same as that of the gaussian case. 
For example, with 1k sample size, the F1 score of LGES under non-gaussianity is 0.79 
while the counter part under the standard setting is 0.82.
%
This is not very unexpected:
the structure identifiability by score is built upon 
the hard constraints imposed by structure on the observational covariance matrix,
which only relies on the linearity of the underlying causal model and does not rely on gaussianity. 
%
As for violation of linearity, we employ  leaky ReLU \citep{xu2015empirical}
to simulate piecewise linear
function, 
as $\mathsf{V_{i}}=\text{LR}(\sum \nolimits_{\mathsf{V_{j}} \in \text{Pa}(\mathsf{V_{i}})} f_{ji} \mathsf{V_{j}} + \epsilon_{\mathsf{V_{i}}})$,
 $\text{LR}(x)=\max(\alpha x,x)$, where $\alpha=0.8$.
The result is in \cref{tab:misspecification}, which shows that LGES  works reasonably well under
certain extent of non-linearity, and still surpasses all baselines. For example, with sample size 1k,
 LGES still achieves 0.71 F1 score even under nonlinearity, while the runner-up achieves 0.65.
 \looseness=-1

\vspace{-0.5em}
\subsection{Real-World Performance}
\vspace{-0.5em}
We consider three real-life datasets.
%
Big Five personality dataset (\url{openpsychometrics.org}),
which consists of 50 questions with 19,719 datapoints. There are five dimensions: Openness, Conscientiousness, 
Extraversion, Agreeableness, and Neuroticism (O-C-E-A-N),
and each dimension with 10 questions (e.g., O1 is the first question for Openness). We use the first 5 questions for each dimension and in total 25 variables.
The structure by LGES is shown in \cref{fig:bigfive}.
Interestingly, without any prior knowledge,
the structure recovered by LGES is very aligned with psychology study.
To be specific, each item in our result is indeed caused by the supposed dimension (latent variable).
Further, we found that some items are caused not only by one latent variable.
For example, E3 (I feel comfortable around people) is caused by L4 (Extraversion), L5 (Agreeableness), and L6 (Neuroticism),
which may shed new light on the existing factor-analysis-based psychometric study.
\looseness=-1

Teacher burnout dataset \citep{byrne2001structural}.
 The term burnout refers to the inability to perform effectively
 in one's job  due to job-related
 stress.
 The dataset includes 32 observed variables with 599 datapoints.
 Multi-tasking behavior dataset
 \citep{himi2019multitasking}. To compare with the model proposed by  \citet{himi2019multitasking},
 we use 9 variables of it with all its 202 datapoints.
 The structures produced by LGES for teacher burnout data and multi-tasking data are shown in \cref{fig:burnout,fig:multitasking} respectively.
%
 Finally, we use three prevalent goodness of fit statistics to validate the structure produced by LGES:
 RMSEA \citep{steiger1980statistically,steiger1990structural}, 
 CFI \citep{bentler1990comparative}, 
 and TLI \citep{tucker1973reliability,bentler1980significance} (detailed in \cref{{sec:appx_goodfit}}).
 We used these fit indices to compare the structure produced by our method with the well-known hypothesized structures in existing psychological studies.
 The result is reported in 
\cref{tab:fitness}, where  the structures by LGES achieve the best scores compared to all the structures proposed in psychological research,
shows that the structures by LGES explain the observational data better than existing study,
and validates the proposed method in real-life scenarios.
\looseness=-1

{ 
\vspace{-0.5em}
\subsection{Comparison With The Exact Search}
\vspace{-0.5em}

In this section we compare the proposed LGES to the score-based exact search method \citep{ngscore}.
The exact search has to enumerate all the possible candidate graphs and thus 
the complexity 
grows super-exponentially with the increase of the number of observed variables $|\set{X}_\graph|$ (which is why we do not include the exact search in our main result in \cref{tab:result}).
%
In contrast to exact search, LGES search over the candidate graph space
in a properly designed way such that 
each time we only need to decide greedily while the asymptotic correctness can still be guaranteed. 

The result is shown in \cref{tab:compare_to_exact_search}.
Specifically, 
to handle a single dataset with 8 observed variables, the exact search requires more than 100 hours while LGES only requires 17 seconds. When it comes to 9 observed variables,
the exact search requires more than 1000 hours (by estimation) while LGES only requires 19 seconds. 
As for the causal discovery performance measured by F1 and SHD, the exact search only performs slightly better than LGES. For example, 
for 7 observed variables, the F1 score of the exact search is 0.89
while the F1 of LGES is 0.88; yet, the the exact search requires nearly 400 times more time (1.5 hours v.s. 14 seconds). This empirical result demonstrates the significant computation efficiency improvement of LGES 
against the exact search, with only slight degradation in the performance.
\looseness=-1
}

\vspace{-0.5em}
\subsection{Implementation Details, Runtime Analysis, and Extendability}
\vspace{-0.5em}
{ 
The readers are referred to \cref{sec:appx_implementation} and \cref{sec:runtime} 
for implementation details
and 
runtime analysis, respectively.
As for the discussion about  the extendability of the proposed method to non-Gaussian or non-linear scenarios, please refer to  \cref{sec:extend_to_nongaussian_nonlinear}. 
}


\vspace{-1em}
\section{Conclusion}
\vspace{-0.5em}
{
In this paper,
we first characterize how likelihood score and minimal dimension are related to the structure identifiability of partially observed linear causal models. Then we  propose Generalized N
Factor Model under which we prove the global consistency of using score for structure identification.
Finally we propose LGES, an asymptotically correct score-based greedy search method
to efficiently search over the graph space and identify the causal structure.}

\section*{Acknowledgments}
We would also like to acknowledge the support from NSF Award No.~2229881, AI Institute for Societal Decision Making (AI-SDM), the National Institutes of Health (NIH) under Contract R01HL159805, and grants from Quris AI, Florin Court Capital, MBZUAI-WIS Joint Program, and the Al Deira Causal Education project.

\bibliography{iclr2026_conference}
\bibliographystyle{iclr2026_conference}

\newpage
\appendix

\begin{algorithm}[tb]
    \caption{Phase 1 of LGES}
    \label{alg:lges_phase1}
   \begin{algorithmic}[1]
    \STATE {\bfseries Input:~} Empirical covariance $\empX$ over $\set{X}$, tolerance $\delta$\;
  \STATE {\bfseries Output:~}{CPDAG $\state_{\text{phase1}}$}
  \STATE Let current state $\state=\state_{\text{init}}$ by Def.~\ref{def:init_state} with latent variables $\set{L}$, 
  let $k=1$, $\set{X}_t=\set{X}$,  
  $\set{L}_d=\emptyset$, and $P$ be an empty list;\;
  \WHILE{$k\leq \frac{|\set{X}_t|}{2}$}
  {
    \FOR{each size $k$ subset $\set{L}_i\subseteq\set{L}_d$ (in parallel)}
    \FOR{each $\node{X}_j\in\set{X}_t$ (in parallel)}
        \STATE $\state_{ij}=\operator_{\set{L}\set{X}}(\state,\set{L}\setminus \set{L}_i,\{\node{X}_j\})$;\;
        \STATE Calculate $\text{score}_{\text{ML}}(\state_{ij},\empX)$;\; 
    \ENDFOR
    \ENDFOR
    \STATE Let $\text{score}_{\text{ML}}(\state,\empX)=s$;\;
    \FOR{each $i,j$ s.t., $|\text{score}_{\text{ML}}(\state_{ij},\empX)-s|\leq \delta$}
        \STATE $\state=\operator_{\set{L}\set{X}}(\state,\set{L}\setminus \set{L}_i,\{\node{X}_j\})$ and $\set{X}_t=\set{X}_t\setminus\{\node{X}_j\}$
    \ENDFOR
    \STATE Let $\set{L}'$ be any size $k$ subset of $\set{L}\setminus\set{L}_d$;\;
    \FOR{each size $k+1$ subset $\set{X}_j\subseteq\set{X}_t$ (in parallel)}
        \STATE $\state_{j}=\operator_{\set{L}\set{X}}(\state,\set{L}\setminus \set{L}',\set{X}_j)$
        \STATE Calculate $\text{score}_{\text{ML}}(\state_{j},\empX)$;\; 
    \ENDFOR

    \STATE 
    Maintain a disjoint set $\setset{X}$ for elements in $\set{X}_j, j\in \{j: |\text{score}_{\text{ML}}(\state_{j},\empX)-\text{score}_{\text{ML}}(\state,\empX)|\leq \delta\}$;\;
    \FOR{each $\set{X}'\in \setset{X}$}
        \STATE Let $\set{L}'$ be any size $k$ subset of $\set{L}\setminus\set{L}_d$;\;
        \STATE $\state=\operator_{\set{L}\set{X}}(\state,\set{L}\setminus \set{L}',\set{X}')$;\;
        \STATE $\set{L}_d=\set{L}_d\cup\set{L}'$, $\set{X}_t=\set{X}_t\setminus\set{X}'$, append $\set{L}'$ to $P$;\;

    \ENDFOR
    \STATE k=k+1;\;
  }
  \ENDWHILE
    \STATE Remove $\set{L}\setminus\set{L}_d$ from $\state$ and let $\state_{\text{phase1}}=\state$;\;
    \STATE {\bfseries return~}{$\state_{\text{phase1}}$, $P$ (which records those latent variables that really exist for the use of Phase 2)}
   \end{algorithmic}
   \end{algorithm}

\begin{algorithm}[tb]
    \caption{Phase 2 of LGES}
    \label{alg:lges_phase2}
   \begin{algorithmic}[1]
    \STATE {\bfseries Input:~} Empirical covariance $\empX$ over $\set{X}$, tolerance $\delta$, the output $\state_{\text{phase1}}$ and $P$ of \cref{alg:lges_phase1}\;
  \STATE {\bfseries Output:~}{CPDAG $\state_{\text{final}}$}
  \STATE Let $\state=\state_{\text{phase1}}$;\;
  \WHILE{True}
  {
    \FOR{$\set{L}_i\in P$ (in parallel)}
    \FOR{$\set{L}_j\in P$ (in parallel)}
        \IF {$\set{L}_i-\set{L}_j$ or $\set{L}_i\rightarrow\set{L}_j$} 
        \STATE Let $\set{H}=\{\node{L}: \node{L}-\set{L}_j \text{~and~} 
        (\node{L}-\set{L}_i \text{~or~} \node{L}\rightarrow\set{L}_i \text{~or~}  \node{L}\leftarrow\set{L}_i)\}$;\;
        \FOR{each subset $\set{H}'_k\subseteq\set{H}$ (in parallel)}
            \STATE $\state_{ijk}=\operator_{\set{L}\set{L}}(S,\set{L}_i, \set{L}_j, \set{H}'_k)$;\;
            \STATE Calculate $\text{score}_{\text{ML}}(\state_{ijk},\empX)$;\; 
        \ENDFOR
        \ENDIF
    \ENDFOR
    \ENDFOR
    \STATE Let $i^*,j^*,k^* = \arg\max \text{score}_{\text{ML}}(\state_{ijk},\empX)$;\;
    \STATE {\bfseries if} {$|\text{score}_{\text{ML}}(\state_{i^*j^*k^*},\empX)-\text{score}_{\text{ML}}(\state,\empX)|\leq\delta$} 
    {\bfseries then} $\state=\state_{i^*j^*k^*}$;\;
    {\bfseries else} Break;\;
  }
  \ENDWHILE
  \FOR{$\set{L}_i\in P$}
    \STATE Delete edges among $\set{L}_i$ in $\state$;\;
  \ENDFOR
    \STATE $\state_{\text{final}}=\state$;\;
    \STATE {\bfseries return~}{$\state_{\text{final}}$}
   \end{algorithmic}
   \end{algorithm}


\section{Detailed Discussion about Related Work}
\label{sec:relatedwork}

In this section, we provide a more comprehensive review of related work based on those discussed in~\cref{sec:intro}, expanding on both latent variable causal discovery and score-based causal discovery.

\paragraph{Latent variable causal discovery:}  
The earliest attempts for handling latent variables in causal discovery were the Fast Causal Inference (FCI) algorithm \citep{spirtes2001causation, richardson2002ancestral, zhang2008completeness}.
{More recently, LCD \citep{rohekar2021iterative}, an iterative causal discovery method, was proposed for structure identification in
the possible presence of latent confounders and selection bias. However, this line of work aims at identifying the Maximal Ancestral Graph. In other words,} the objective is more of ``deconfounding'' for the identification of causal relations among observed variables, and does not provide direct insight into the causal structure among latent variables. Moreover, FCI is already proved to be maximally informative under nonparametric conditional independence constraints. Therefore, to go beyond the limitations of conditional independence constraints, new statistical tools have been developed, often relying on additional parametric assumptions.  

These tools have been discussed in~\cref{sec:intro}. Among them, the most commonly used one and also the earliest developed one might be rank constraints \citep{ sullivant2010trek}, which generalize the classical Tetrad representation theorem~\citep{spirtes2001causation} and basic conditional independence constraints. Based on these new tools, many algorithms have also been developed~\citep{silva2003learning,huang2022latent,dong2023versatile}, with tests that handle discretizations \citep{dong2025permutation,sun2025sample}. However, despite their theoretical advancements, most existing methods remain within the constraint-based paradigm, heavily relying on statistical tests that suffer from multiple-testing and error propagation issues. This is exactly our motivation on developing score-based algorithms for latent variable causal discovery.

\paragraph{Score-Based causal discovery:} Score-based methods offer an alternative to constraint-based approaches, mitigating some of the issues related to error propagation. These methods search for an optimal structure by maximizing a scoring function, such as the Bayesian Information Criterion (BIC) or likelihood-based scores. Based on their search strategies, they can be broadly categorized as follows:

\begin{itemize}[leftmargin=*, itemsep=0pt]
    \item Exact Search Methods employ exhaustive graph traversal techniques or exploit minimal pruning strategies, such as permutation search~\citep{raskutti2014learning}, dynamic programming~\citep{koivisto2004exact}, or integer linear programming \citep{, cussens2012bayesian}, to identify the globally optimal structure. These methods require minimal assumptions about the underlying graph structure or parametric model. However, their computational complexity grows super exponentially with the number of variables, rendering them impractical for large-scale causal discovery. 
    \item A*-based and heuristic search methods integrate heuristic functions into the search process to guide exploration through the graph space \citep{yuan2013learning,scanagatta2015learning}. These methods strike a balance between computational efficiency and search completeness by prioritizing graph structures with high potential scores while avoiding exhaustive enumeration. Although more scalable than exact search, the quality of the learned structure relies heavily on the effectiveness of the heuristic function.  
    \item Greedy search methods, such as the widely used Greedy Equivalence Search (GES) \citep{chickering2002optimal}, formulate the search problem in terms of graphical operators that iteratively modify the structure by adding, deleting, or reversing edges. These methods are computationally efficient and well-suited for large-scale problems. As illustrated in \citep{nandy2018high}, greedy search methods can often be interpreted as progressively refining the graph structure based on conditional independence or other graphical constraints, offering an intuitive connection between constraint-based and score-based paradigms.
    \item Differentiable approaches, such as the seminal NOTEARS method \citep{zheng2018notears}, recast the structure learning problem as a continuous optimization task. Subsequent works have incorporated nonlinear functional forms~\citep{yu19daggnn,lachapelle2020grandag,zheng2020learning,ng2022masked}, interventional data~\citep{brouillard2020differentiable,faria2022differentiable,lippe2022enco}, alternative optimization techniques~\citep{ng2020role,ng2022convergence,bello2022dagma,deng2023optimizing}, and improved formulations of the acyclicity constraint~\citep{yu19daggnn,lee2019scaling,wei2020nofears,bello2022dagma,zhang2022truncated,zhang2025analytic}. These methods benefit from direct compatibility with well-established numerical solvers and GPU acceleration, enabling them to efficiently handle large-scale problems.
\end{itemize}

Lastly, let us note that in the intersection of latent variable causal discovery and score-based algorithms, the only existing approach, to our knowledge, is that of \citet{ngscore}. As for the specific search procedure, their method follows an inefficient exact search paradigm. In contrast, our work is the first to introduce greedy score-based search for latent causal discovery, offering a more practical and scalable solution to real-world problems while maintaining identifiability guarantees.
{We note that \cref{sec:algebraic_equivalence} is highly related to \citet{ngscore},
and yet our contribution in \cref{sec:algebraic_equivalence} is unique,
with reasons as follows.

First,
to establish structure identifiability based on likelihood score and dimension,
we have to assume the generalized faithfulness (the spirit of which is similar to the classical CI faithfulness). To assume the generalized faithfulness, one has to formally define several necessary notions including equality constraints, 
$H(\graph)$ (the set of equality constraints for a graph $\graph$),
$B(\graph)$ (the set of canonical  equality constraints with respect to reduced Grobner basis), and $\mathbb{H}^n$ (the union of $B(\graph)$ over all possible graphs with $n$ observed variables). 
This paper provides formal definitions of these crucial notions,
which 
serves as the 
basis for a rigorous version of the generalized faithfulness \cref{assumption:generalized_faithfulness};
as a contrast, these notions are not rigorously defined in \citet{ngscore},
which gives rise to some counter examples that shows the violation set of 
the generalized faithfulness might not be of measure zero.
%

Second, the identifiability theory proposed in \citet{ngscore} is based on $\text{score}_{\text{dim}}$, where the likelihood and dimension are entangled.
This actually requires a score based method (either exact search or greedy search)
to be able to explicitly characterize the dimension of any candidate graph
during the search procedure.
However, characterizing an arbitrary latent causal structure still remains an open challenge in the field, which limits the usage of the theory.
In contrast, this work explicitly disentangled the maximum likelihood score and dimension when stating the identifiability theory in \cref{thm:equivalence_equality_constraints}. This gives rise to our greedy search method that does not need to explicitly characterize the dimension of each candidate graph during the search.
}

{This work is also closely related to the RLCD algorithm \citep{dong2023versatile}.
Specifically, both RLCD and LGES aims to identify the underlying causal structure involving latent variables given observational distribution.
The RLCD algorithm, which is a constraint based method, makes use of rank constraints.
Note that the set of all rank constraints is a super-set of the vanishing partial correlation constraints and a  sub-set of all the equality constraints, and RLCD examines the rank constraints by using statistical rank tests. In contrast, LGES is a score based method, which does not rely on statistical tests. In essence, by comparing the likelihood scores and dimensions, LGES implicitly makes use of all the equality constraints, which is a super-set of the rank constraints.
}

The  research on nested Markov model 
\cite{shpitser2012parameter,shpitser2014introduction,richardson2023nested} is related to causal discovery in the presence of latent variables.
This line of work is elegant in that it accomodates Verma-type constraints and contains marginal distributions given by a DAG model with latent variables. However, nested Markov models follow the acyclic directed mixed graph (ADMG) framework, where the effect of latent variables are simplified into bidirected edges between observed variables. That is to say, within ADMG, only structure among observed variables is concerned. On the contrary, this paper aims to identify the whole underlying causal structure among both observed and latent variables (e.g., an edge from a latent variable to another latent variable). The research on learning
phylogenetic tree \cite{felsenstein2004inferring,huelsenbeck2001bayesian} is also related as it aims to infer the evolutionary relationships among a group of organisms using observed data—typically and outputs a tree-structured graph. Yet, the graphical assumption in this line of work is much stronger then the GNFM considered in this paper.

\begin{center}
  \begin{table*}[tb]
  \vspace{-1em}
    \caption{{ Comparison of LGES with exact search across different graph sizes with 1000 sample size.
      - means the result is unavailable due to that the complexity of exact search grows super-exponentially.}}
     \vspace{-0.5em}
     \label{tab:compare_to_exact_search}
    \footnotesize
    \center 
  \begin{center}
  \resizebox{\linewidth}{!}{
  { 
  \begin{tabular}{|c|c|c|c|c|c|c|c|}
    \hline  \multicolumn{1}{|c|}{} &\multicolumn{2}{c|}{{F1 score for skeleton $\uparrow$} }&\multicolumn{2}{c|}{{SHD for MEC $\downarrow$} }&\multicolumn{2}{c|}{{Time to find a single graph $\downarrow$}}\\
    \hline 
    \multicolumn{1}{|c|}{Graph Size}  & {LGES} & {Exact Search} & {LGES} & {Exact Search} & {LGES} & Exact Search\\
    \hline 
    $|\set{X}_{\graph}|=5$ & \textbf{0.96} (0.06) & \textbf{0.96} (0.03) & {0.55} (0.43)&\textbf{0.37} (0.05) &\textbf{9 seconds} & 15 seconds  \\
    \hline 
    $|\set{X}_{\graph}|=6$ & \textbf{0.91} (0.02) & \textbf{0.91} (0.02) & 1.24 (0.75) &  \textbf{0.98} (0.54) &\textbf{11 seconds} & 8 minutes  \\
    \hline
    $|\set{X}_{\graph}|=7$ & {0.88} (0.03) & \textbf{0.89} (0.02) & 2.20 (0.87)& \textbf{2.05} (0.80) &\textbf{14 seconds} & 1.5 hours  \\
    \hline
    $|\set{X}_{\graph}|=8$ & {0.86} (0.02) & - & 3.56 (0.78)& - & \textbf{17 seconds} & $>$100 hours \\
    \hline
    $|\set{X}_{\graph}|=9$ &  {0.85} (0.03) & - &  4.11(1.52)& - & \textbf{19 seconds} & $>$1000 hours \\
    \hline
  \end{tabular}
  }
  }
  \end{center}
  \vspace{-0.5em}
  \end{table*}
  \end{center}

\begin{table}[t]
\hspace{0em}
\begin{minipage}[c]{0.5\linewidth}
\centering
     \vspace{-0em}
    \caption{Ablation study on the sensitivity of hyper-parameter $\delta$.}
    \label{tab:on_delta}
    \vspace{0em}
   \resizebox{\linewidth}{!}{
  \begin{tabular}{|c|c|c|}
    \hline  \multicolumn{1}{|c|}{value of $\delta$} &\multicolumn{1}{c|}{{F1 score for skeleton $\uparrow$} }&\multicolumn{1}{c|}{{SHD for MEC $\downarrow$} }\\
    \hline 
    $5\times10^{-4}$ & {0.79}(0.02) & 10.85(0.79)\\
    \hline 
    $1\times10^{-3}$ & {0.81}(0.02) & 9.52(0.84)\\
    \hline
    $2\times10^{-3}$ & \textbf{0.82}(0.02) & \textbf{8.80}(0.70)\\
        \hline 
    $4\times10^{-3}$ & {0.80}(0.02) & 9.85(0.59)\\
        \hline 
    $1\times10^{-2}$ & {0.78}(0.02) & 11.13(0.70)\\
        \hline 
  \end{tabular}
  }
  \end{minipage}
\end{table}

\section{Proofs}
\subsection{Proof of \cref{thm:equivalence_equality_constraints}}
\TheoremEquivalenceEqualityConstraints*

The overall proof strategy below is inspired by \citet{ghassami2020characterizing,ngscore}.
\begin{proof}
In the large sample limie, we have $\Sigma_{\set{X}}^*=\empX$.
By $\hat{\graph}\in\mathbb{G}^*$,
        we have that 
        $\hat{\graph}$ can generate $\Sigma_{\set{X}}^*$
        and thus 
  $\Sigma_{\set{X}}^*$ contains all the equality and inequality constraints of $\hat{\mathcal{G}}$. Under the generalized faithfulness assumption, we have 
\begin{equation}\label{eq:equalities_subset}
H(\hat{\mathcal{G}})\subseteq H(\mathcal{G}^*).
\end{equation}
      Further, we have $\hat{\graph}\in \argmin_{\graph\in\mathbb{G}^*} \dim(\graph)$.
      As $\graph^*\in\mathbb{G}^*$,
      we have $\dim(\hat{\graph})\leq\dim(\graph^*)$.
      Suppose by contradiction that $H(\hat{\mathcal{G}})\subsetneq H(\mathcal{G}^*)$. This implies $\dim(\hat{\mathcal{G}})>\dim(\mathcal{G}^*)$, which contradicts with $\dim(\hat{\graph})\leq\dim(\graph^*)$. Therefore, we have 
\begin{equation}\label{eq:equalities_not_subset}
H(\hat{\mathcal{G}})\not\subsetneq H(\mathcal{G}^*).
\end{equation}
Taking \cref{eq:equalities_subset,eq:equalities_not_subset} together, we have $H(\hat{\mathcal{G}})= H(\mathcal{G}^*)$.  
\end{proof}

\subsection{Proof of \cref{thm:identify_gnfm}}
\TheoremIdentifiabilityGNFM*

\begin{proof}
We prove by showing that, by using equality constraints we can identify a graph $\graph\in\mathbb{G}_{\text{GNFM}}$ up to MEC.

  Suppose $\graph$ satisfies \cref{definition:gnfm}
   and thus 
there exists a partition of all latent variables in $\graph$ that satisfies the requirement in \cref{definition:gnfm}, as $\{\set{L}_i\}_1^P$.
For each $\set{L}_i$, there exist at least $|\set{L}_{i}|*2$ observed variables $\set{X}_{i}$
such that for all $\node{X}\in\set{X}_{i},
\parents(\node{X})=\set{L}_{i}$.
We first prove by induction that the structure from latent variables to observed variables can be identified up to MEC by equality constraints.
Let $k=1$. For those $|\set{L}_i|=k$,
all the pure children of $\set{L}_i$ can be identified by rank constraints:
We can simply check size $k+1$ combination of observed variables $\hat{\set{X}}$ and we have the variables in $\hat{\set{X}}$ are pure children of the same size 1 latent group, iff $\text{rank}_{\Sigma_{\hat{\set{X}},
\set{X}_\graph\setminus\hat{\set{X}}}}=k$, given the relation between rank and t-separation in \citet{sullivant2010trek}.
Next, suppose when
the pure children of all the latent groups with $|\set{L}_i|\leq k$
 have been found, 
 we show the pure children of  latent groups with $|\set{L}_i|=k+1$ can also be found by rank constraint.
In this case,  check all size $k+2$ combination of observed variables $\hat{\set{X}}$ such that $\text{rank}_{\Sigma_{\hat{\set{X}},
\set{X}_\graph\setminus\hat{\set{X}}}}=k+1$. Consider all such  $\{\hat{\set{X}}\}_i^c$ and maintain a disjoint set for them such that
 $\hat{\set{X}}_i$ and  $\hat{\set{X}}_j$ belong to the same group if they have at least one common element.
Take the union of such a group, say $\tilde{\set{X}}$.
If $\tilde{\set{X}}$ share no common element with the pure children of any $\set{L}_i$ that has size$\leq k$, then elements in $\tilde{\set{X}}$ are pure children of the same latent group $\set{L}_j$ with $|\set{L}_j|=k+1$.
By induction, all the pure children of each latent group can be identified.

Further, for an observed variable, say, $\node{X}$, that is a common child of multiple latent groups, suppose its common parents are $\setset{L}$,
which is a set of some groups of latent variables.
Let 
 $\set{A}=\node{X}\cup\bigcup_{\set{L}_j\in\setset{L}}\{\purechildren^1(\set{L}_j)\}$, where $\purechildren^1(\set{L}_j)$
  is a set of $|\set{L}_j|$ pure children of $\set{L}_j$,
 we have $\text{rank}_{\Sigma_{\set{A},
\set{X}_\graph\setminus\set{A}}}=\sum_{\set{L}_j\in\setset{L}}|\set{L}_j|$, 
and that when any edge related to $\node{X}$ is changed, the related observed rank constraint will change,
and thus the set of equality constraints will also change.
Therefore, we have that if  $H(\graph_1)=H(\graph_2)$, then the structure from latent to observed variables in $\graph_1$ and that in $\graph_2$ are the same.

Next, we show that the structure among latent groups 
can be identified up to MEC by equality constraints.
   By making use of Corollary 1.3 in \cite{di2009t},
   we can translate d-separation between latent groups into t-separation among the pure children of these latent groups.
   Specifically,  for $\set{L}_i,\set{L}_j, i\neq j$,
   we have $\set{L}_i,\set{L}_j$,
    are d-separated by $\setset{L} \subseteq \{\set{L}_l\}_1^P \backslash \{\set{L}_i,\set{L}_j\}$,
    iff $\set{A}=\purechildren^1(\set{L}_i)\cup\bigcup_{\set{L}_l\in\setset{L}}\{\purechildren^1(\set{L}_l)\}$ and
    $\set{B}=\purechildren^1(\set{L}_j)\cup\bigcup_{\set{L}_l\in\setset{L}}\{\purechildren^2(\set{L}_l)\}$
    are t-separated by $\{\setset{L},\emptyset\}$ or $\{\emptyset,\setset{L}\}$,
    where $\purechildren^1(\set{L}_l)$ and $\purechildren^2(\set{L}_l)$ refer to two disjoint groups
     of $|\set{L}_l|$ pure children of $\set{L}_l$ (by definition we know such two groups must exist).
    Given \cref{assumption:generalized_faithfulness} and the relation between rank and t-separation in \citet{sullivant2010trek}, we have that
    $\set{L}_i$ and $\set{L}_j$ are d-separated by $\setset{L}$,
    iff $\rank_{\Sigma_{\set{A}, \set{B}}}=||\setset{L}||$,
    where $||\setset{L}||=|\bigcup_{\set{L}_l\in\setset{L}}\set{L}_l|$.
    This means that the d-separations among latent groups can be inferred from rank constraints on observed variables,
    and d-separations can be used to identified the structure among latent variables up to MEC.
    As rank constraints are part of the equality constraints, 
    we have that 
    the structure among latent variables
    can be identified up to MEC by equality constraints.
    Taking that the structure from latent to observed variables can also be identified up to MEC by equality constraints (proved before)
    and in Generalized N factor models there is no direct edge between observed variables,
    we have 
    that if $H(\graph_1)=H(\graph_2)$, then $\graph_1$ and $\graph_2$ belong to the same MEC.
\end{proof}


\subsection{Proof of \cref{corollary:global_consistency}}
\CorollaryGlobalConsistency*
\begin{proof}
  Similar to the proof of \cref{thm:equivalence_equality_constraints},
  we can show that $H(\hat{\graph})=H(\graph^*)$.
  By \cref{thm:identify_gnfm}, we have $\hat{\graph}$
  and $\graph^*$ belong to  the same MEC.
\end{proof}

\subsection{Proof of \cref{lemma:init_can_generate}}
\LemmaInitCanGenerate*
\begin{proof}
By the graphical assumption in \cref{definition:gnfm},
suppose $\set{L}$ is the set of all latent variables in $\graph^*$,
then there must exist at least $|\set{L}|\times 2$ observed variables in $\graph^*$.
As such, 
the number of latent variables $|\set{L}|$ should be no more than  $\lfloor \frac{|\set{X}|}{2}\rfloor$. 
Thus, by \cref{def:init_state}, $\state_{\text{init}}$ must be a supergraph of the CPDAG of the ground truth, 
$\state(\graph^*)$, up to permutation of latent variables,
and thus
$\state_{\text{init}}$ must be able to generate the observation
equally well as  $\state(\graph^*)$.
\end{proof}

\subsection{Proof of \cref{lemma:correctphase1}}
\LemmaCorrectnessPhaseone*

We provide a sketch of proof as follows.

\begin{proof}
We prove by induction.
First consider $k=1$. We show that 
all the observed variables that belong to a size 1 latent group can be identified.
Specifically,  $\operator_{\set{L}\set{X}}(\state,\set{L}\setminus \set{L}',\set{X}_j)$ in line 17 introduces a rank constraint $\text{rank}({\Sigma}_{\set{X}_j,\set{X}\setminus\set{X}_j})=1$ to $\state_j$,  if $\state_j$ in line 17 of \cref{alg:lges_phase1} can still generate the observation,
then under the generalized faithfulness in \cref{assumption:generalized_faithfulness}  this rank constraint must also belong to $H(\graph^*)$, which means that variables 
in $\set{X}_j$ must belong to the same size $1$ latent group in $\graph^*$.
As such, all observed variables that belong to the same size $1$ latent group can be identified.
Now, suppose we have already identified those observed variables whose parent set has cardinality $t$ and 
let the set of these observed variables as $\set{X}_{\text{done}}^t$, for all $t\leq T$.
We show that for $k=T+1$, 
all the observed variables that belong to a size $k$ latent group can be identified.
Specifically, $\state_j$ introduces a rank constraint $\text{rank}({\Sigma}_{\set{X}_j,\set{X}\setminus\set{X}_j})=k$ to $\state_j$,  if $\state_j$ in line 17 of \cref{alg:lges_phase1} can still generate the observation,
then under the generalized faithfulness in \cref{assumption:generalized_faithfulness} this rank constraint must also belong to $H(\graph^*)$.
Given that  $\set{X}_j$ has no common variable with any $\set{X}_{\text{done}}^t$ for all $t\leq T$,
we have that $\set{X}_j$ must belong to the same size $k$ latent group in $\graph^*$.
Therefore, by the end of the while in line 27 in \cref{alg:lges_phase1},
all the parents of each observed variable can be identified.
Till now,  in the state considered in the algorithm, there might still exist
some latent variables having no observed variable as children.
These latent variables are removed from the current state in line 28 in \cref{alg:lges_phase1},
and thus the number of latent variables in $\state_{\text{phase1}}$ is the same as that of the ground truth and the structure between latent variables and observed variables in $\state_{\text{phase1}}$ is also the same as that of the ground truth up to permutation of latent variables.
\end{proof}

\subsection{Proof of \cref{theorem:correctphase2}}
\TheoremCorrectnessPhaseTwo*
The proof is partially inspired by the proof of Lemma 10 in \citet{chickering2002optimal}.
\begin{proof}
First, we know that all the states must be able to generate the observation.
Assume that Phase 2 terminates with a sub-optimal state $\state'$
and let $\graph'$ be a DAG that belongs to $\state'$.
By Theorem 4 in \citet{chickering2002optimal}
we know that there must exist a sequence of covered edge reversals and edge additioins that transforms $\graph^*$ to $\graph'$. Suppose $\graph''$ precedes the last edge addition in the sequence. We have that $\graph''$ must also be able to generate the observation and $\state(\graph'')$ is a neighboring state of $\state'$, 
which means Phase 2 should not terminate at $\state'$, yielding a contradiction.
Thus, by the end of Phase 2, the structure among latent variables in $\state_{\text{final}}$ must be the same as that of the ground truth
up to permutation of latent variables.
Taking \cref{lemma:correctphase1} also into consideration, we have that 
  in the large sample limit the output $\state_{\text{final}}$ of \cref{alg:lges_phase1,alg:lges_phase2} is a CPDAG 
    that represent the MEC of $\graph^*$,
    up to permutation of latent variables.  
\end{proof}

\section{Additional Definitions, Implementation Details, Runtime Analysis, and Examples}

\subsection{{Detailed Definition of Equality Constraints, Inequality Constraints, $H(\graph)$, and  $\mathbb{H}^n$.}}
\label{sec:def_HG}

\begin{definition}[{Definition of Equality Constraints,
Inequality Constraints,
and $H(\graph)$.}]
  \label{def:HG}
  Let $\mathcal{G}$ be the DAG structure of a Partially Observed Linear Causal Model 
  with $m$ latent variables and $n$ observed variables (as in \cref{definition:polcm}).
The entries of the observed covariance matrix $\Sigma_{\mathbf{X}}$ are 
polynomial functions of model parameters $\theta = (F^T,\Omega_{\epsilon})$, 
where $F^T$ is the edge coefficient matrix and $\Omega_{\epsilon}$ is the diagonal noise variance matrix. This induces a parametric map under $\mathcal{G}$:
$$\phi_{\mathcal{G}}:\mathbb{R}^{|\theta|}\rightarrow\mathbb{R}^{n(n+1)/2},$$
from parameters to observed covariance matrix,
defining the system of equations: $$\{ {\Sigma_{\mathbf{X}}}_{ij}-\phi_{\mathcal{G},ij}(\theta)=0| 1\leq i \leq j \leq n\}.$$
{Each equality constraint is 
just a polynomial equality equation consists of entries of $\Sigma_{\set{X}}$ and $\theta$,
and each inequality constraint is just  a polynomial inequality equation consists of entries of $\Sigma_{\set{X}}$ and $\theta$.
As our objective is to find information from $\Sigma_{\set{X}}$ to infer the causal structure,
we should focus on  those  equality equations that  consists of only entries of $\Sigma_{\set{X}}$, which can be achieved as follows.
}

 Let $\mathbb{R}[\theta,\Sigma_{\mathbf{X}}]$ be the polynomial ring 
 which contains all 
variables for all model parameters $\theta$ and all distinct covariance entries ${\Sigma_{\mathbf{X}}}_{ij}$,
while $\mathbb{R}[\Sigma_{\mathbf{X}}]$ be the polynomial ring on $\Sigma_{\mathbf{X}}$ only.
Define the ideal $I_{\mathcal{G}}\subseteq
  \mathbb{R}[\theta,\Sigma_{\mathbf{X}}]$ generated by the above equations as:
  $$I_{\mathcal{G}}= \langle \{ {\Sigma_{\mathbf{X}}}_{ij}-\phi_{\mathcal{G},ij}(\theta)=0| 1\leq i \leq j \leq n\} \rangle.$$
Then, $H(\mathcal{G})$ is the elimination ideal obtained by intersecting $I_{\mathcal{G}}$ with $\mathbb{R}[\Sigma_{\mathbf{X}}]$, i.e.,
$$H(\mathcal{G}):=I_{\mathcal{G}}\cap\mathbb{R}[\Sigma_{\mathbf{X}}].$$

{
In other words, $H(\mathcal{G})$ contains all equality constraints implied by $\mathcal{G}$ on the observed covariance matrix.
}

\end{definition}

\begin{definition}[Definition of $B(\mathcal{G})$  and $\mathbb{H}^n$.]
  \label{def:BGHn}
Let $>$ be a fixed lexicographic monomial ordering on $\mathbb{R}[\theta, \Sigma_{\mathbf{X}}]$ such that all parameter variables in $\theta$ are greater than all covariance variables in $\Sigma_{\mathbf{X}}$.
Define $B(\mathcal{G})$ as follows.
 (i) Compute the reduced Gröbner basis of $I_{\mathcal{G}}$ following ordering $>$, i.e., $G_{B}(I_{\mathcal{G}},>).$
 (ii) 
 Retain only those polynomials that involve only variables in $\Sigma_{\mathbf{X}}$.
 Formally,
$$B(\mathcal{G}):=G_{B}(I_{\mathcal{G}},>)\cap\mathbb{R}[\Sigma_{\mathbf{X}}].$$
Then $\mathbb{H}^n$ is defined as $$\mathbb{H}^n:=\bigcup_{\mathcal{G}\in\mathbb{G}^n} B(\mathcal{G}).$$

\end{definition}

In essence, 
$H(\mathcal{G})$ contains all equality constraints implied by $\mathcal{G}$ on the observed covariance matrix,
 while $B(\mathcal{G})$ consists of a canonical and minimal (owing to reduced Gröbner basis)
 set of polynomial constraints among the observed covariances that must vanish for any distribution consistent with
  structure $\mathcal{G}$. Since the reduced Gröbner basis is unique (given a fixed monomial order), $B(\mathcal{G})$ serves as a standard representative of these constraints. The vanishing set of $B(\mathcal{G})$ defines the smallest algebraic variety that contains all observed covariance matrices generated by the model.

\subsection{Definition of One Factor Model by \citet{silva2003learning} and Comparison}
\label{sec:appx_one_factor_model}

\begin{definition}[One Factor Model \citep{silva2003learning}]
\label{definition:silva_graphical_criterion}
 DAG $\graph$ satisfies the definition of One Factor Model if
each measured variable has a single latent parent, and each latent variable has at least three measured variables as children.
\end{definition}

First, the generalized N factor model takes one factor model as a special case. An example of one factor model can be found in \cref{fig:compare_gnfm_ofm} (b). As a comparison, an example of generalized N factor model can be found in \cref{fig:compare_gnfm_ofm}
(a).
Specifically, (a) differs from (b) in that, (i) (a) allows latent variables to form a group and share observed variables as children, 
e.g., $\{\node{L}_4,\node{L}_5\}$ in (a) compared to $\node{L}_4$ in (b), (ii) (a) allows some observed variables to be common children of muiltiple groups of latent variables, e.g., $\node{X}_{14}$ in (a),
while (b) does not.

{
\subsection{Definition of V-Structure and CPDAG}
\label{sec:def_vstructure_cpdag}

\begin{definition}[V-Structure / Collider]
\label{def:vstructure}
Let $\mathcal{G} = (\set{V}, \set{E})$ be a directed graph, where $\set{V}$ is a set of vertices (nodes) and $\set{E}$ is a set of directed edges.
A v-structure is an ordered triplet of distinct nodes $(\node{X}, \node{Y}, \node{Z})$ where $\node{X}, \node{Y}, \node{Z} \in \set{V}$, that satisfies the following two conditions:
\begin{itemize}[leftmargin=*, itemsep=0pt]
    \item The graph $\mathcal{G}$ contains the directed edges $\node{X} \rightarrow \node{Y}$ and $\node{Z} \rightarrow \node{Y}$.
    \item $\node{X}$ and $\node{Z}$ are not adjacent in $\mathcal{G}$.
\end{itemize}
The node $\node{Y}$ is referred to as the collider node of the v-structure.
\end{definition}

\begin{definition}[CPDAG (also called  essential graph or maximally oriented graphs) \citep{spirtes2001causation,chickering2002optimal}]
\label{def:cpdag}
Let $\mathcal{E}$ be the Markov equivalence class of a DAG
$\mathcal{G}$.
The CPDAG that represents $\mathcal{E}$ is the unique graph 
$\mathcal{G}^*$ (consisting of both directed and undirected edges) 
such that:

\begin{itemize}[leftmargin=*, itemsep=0pt]
    \item \textbf{Directed Edge:} An edge $\node{X} \rightarrow \node{Y}$ is in $\mathcal{G}^*$ if and only if the edge $\node{X} \rightarrow \node{Y}$ exists in every DAG $\mathcal{G}_1 \in \mathcal{E}$.
    
    \item \textbf{Undirected Edge:} An edge $\node{X} - \node{Y}$ is in $\mathcal{G}^*$ if and only if there exists at least one DAG $\mathcal{G}_1 \in \mathcal{E}$ containing the edge $\node{X} \rightarrow \node{Y}$ and at least one other DAG $\mathcal{G}_2 \in \mathcal{E}$ containing the edge $\node{X} \leftarrow \node{Y}$.
\end{itemize}
\end{definition}
}

\subsection{Goodness-of-model-fit Measures}
\label{sec:appx_goodfit}
RMSEA \citep{steiger1980statistically,steiger1990structural} measures the discrepancy
due to the approximation per degree of freedom. It is actually a badness-of-fit measure and thus the lower value the better fit of the model. The sample RMSEA is estimated as follows.
\begin{align}
\label{eq:rmsea}
\hat{\text{RMSEA}}=\sqrt{\frac{\max(\chi^2-{df},0)}{df(N-1)}},
\end{align}
where $\chi^2$ is the chi-square statistic 
of the concerned model and
 $df$ is 
the degree of freedom of the chi-square statistic.

The CFI \citep{bentler1990comparative} measures the relative improvement in terms of fit from the baseline model to the proposed model. The sample CFI is estimated as follows:
\begin{align}
\label{eq:cfi}
\hat{\text{CFI}}=1-\frac{\max(\chi_k^2-{df}_k,0)}{\max(\chi_0^2-{df}_0,0)},
\end{align}
where 
$\chi^2_k$ and $df_k$ corresponds to the concerned model
while $\chi^2_0$ and $df_0$ corresponds to the baseline independent model that 
can only parameterize a diagonal covariance matrix.

The TLI \citep{tucker1973reliability,bentler1980significance} measures a relative reduction in misfit per degree of freedom. The sample estimator of TLI can be given as follows:
\begin{align}
\hat{\text{TLI}}=
\frac{\chi_0^2/{df}_0 - \chi_k^2/{df}_k }{\chi_0^2/{df}_0 -1}.
\end{align}

\begin{figure*}[t]
  \vspace{-0mm}
\setlength{\belowcaptionskip}{0mm}
  \centering
    \subfloat[An illustrative graph that satisfies generalized N factor model.]
    {
    \centering
    \includegraphics[width=0.45\textwidth]{figures/example_gnfm.png}
    }
    \centering
    \subfloat[An illustrative graph that satisfies one factor model.]
    {
    \centering
    \includegraphics[width=0.45\textwidth]{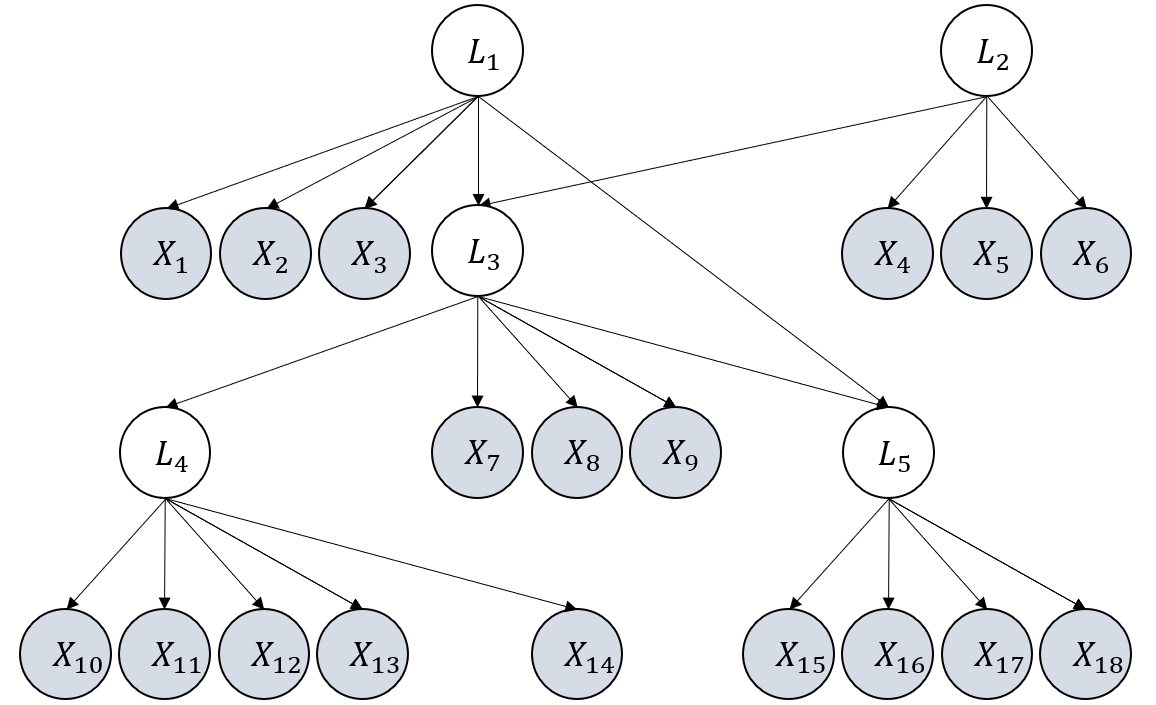}
    }
\caption{
Illustrative examples to compare two graphical assumptions, generalized N factor model v.s. one factor model.
 }
  \label{fig:compare_gnfm_ofm}
  \vspace{-1em}
\end{figure*}

\subsection{Implementation Details and Discussion on the design of $\delta$}
\label{sec:appx_implementation}
\vspace{-0.5em}
Our code is based on Python3.7 and PyTorch \citep{paszke2017automatic}
and the optimization problem in \cref{eq:ml} is solved by  
Adam~\citep{kingma2014adam} and LBFGS.
Data is standardized to have zero mean and unit variance.
The hyper parameter $\delta$ in \cref{alg:lges_phase1,alg:lges_phase2}
is set as $\delta=0.25\times\frac{\log(N)}{N}$, where $N$ is the sample size. This design follows the spirit of BIC score such that
$\delta\rightarrow0$ when $N\rightarrow\infty$. 
In practice,
we found that the result is only influenced marginally by a small change of $\delta$.

Bellow we provide a further discussion on  $\delta$ and the criterion of keeping edge removal in LGES.

Specifically, in LGES an edge removal is kept only when $|\text{LogL}\_{\text{curr}}-\text{LogL}\_{\text{prev}}|\leq k\frac{\log N}{N}$, where $\delta\coloneqq k\frac{\log N}{N}$ is the tolerance level, $N$ is the sample size, and $k$ is a hyper-parameter that controls sparsity.
This design follows the spirit of BIC score in GES ($\text{score}\_\text{BIC}=\text{LogL} - 0.5\frac{\log N}{N} \text{dim}$). Thus, the behavior of LGES is quite similar to GES - both of them accommodates asymptotic consistency and finite sample performance at the same time.

In the asymptotic case, $0.5\frac{\log N}{N} \text{dim}\rightarrow 0$, and thus the likelihood term dominates the BIC score. Therefore, what
GES favors is exactly the graph that (i) 
 has the best likelihood, (ii) at the premise of (i) the dimension should be as small as possible. 
Similarly, in the asymptotic case $\delta\rightarrow0$, and thus in LGES what is performed is precisely "Only when the likelihood after the deletion is still the best do we keep the deletion" in the asymptotic case. Combined with other designs as discussed in our last response, 
LGES also achieves the goal of finding the graph that (i) has the best likelihood, (ii) at the premise of (i) the dimension should be as small as possible, and thus guarantees the asymptotic consistency.

In the finite sample case, there exists a problem that a supermodel always has a better likelihood. To address this problem, in GES the BIC score sensibly encourages an edge removal by the term $0.5\frac{\log N}{N} \text{dim}$. As in GES each edge removal results in exactly 1 dimension decrease, the criterion in GES is equivalent to keeping an edge removal as long as $\text{LogL}\_{\text{curr}} \geq \text{LogL}\_{\text{prev}}  - 0.5\frac{\log N}{N}$. In LGES, an edge removal is kept when $ |\text{LogL}\_{\text{curr}}-\text{LogL}\_{\text{prev}}|\leq k\frac{\log N}{N}$, and thus it is also encouraged sensibly in LGES with finite samples. 

\subsection{Runtime Analysis}
\label{sec:runtime}
\vspace{-0.5em}

 Next we discuss the time complexity of LGES. Similar to the classical PC and GES, our method has a worst-case complexity exponential in the number of observed variables. However, if the underlying graph is sparse, which is a common and reasonable assumption \cite{kalisch2007estimating}, the complexity becomes polynomial.
 The intuition is as follows. 
 Similar to PC and GES, during the process, 
 LGES enumerates different combinations of variables and check the score to decide whether to
  delete some edges. Although the number of all combinations is exponential (also in GES and PC), 
  if the underlying graph is sparse, e.g., 
  maximum degree of a node is P, the algorithm will successfully find 
  the correct combination to delete the edge before enumerating all the combinations.
   Thus the number of combinations that are actually enumerated only depends on the constant
    P instead of number of variables. Therefore, the time complexity will become a term polynomial 
    in N, times a term polynomial in constant P, and thus polynomial in N.

 In our implementation, the computational cost is almost irrelevant to sample size, as we only need to calculate the sample covariance once and cache it. Further, lines 5,6,13 in Alg 1 and lines 5,6,9 in Alg 2 can be executed in parallel. Owe to these designs, in practice LGES is fairly efficient: on average it takes only one minute to handle a graph with 20 variables. 

We conduct all the experiments with single Intel(R) Xeon(R) CPU E5-2470.
Thanks to the fact that lines 5,6,13 in \cref{alg:lges_phase1}
and lines 5,6,9 in \cref{alg:lges_phase2} can be executed in parallel,
we employ the package joblib in python to conduct them by multi-processing.
As such,
on average it takes only one minute to handle a graph with 20 variables.
%
We note that the computational cost is almost irrelevant to sample size, as we only need to calculate the sample 
covariance matrix once and cache it for further use. Compared to RLCD and GIN, 
LGES is faster than RLCD but slower than GIN. 
Specifically, it takes RLCD around 2 minutes and GIN around 20 seconds to handle a graph with 20 variables,
while around one minute for LGES.
The reason why GIN is faster than LGES, 
is that GIN only focuses on the structure between latent variables and observed variables
 and does not identify structure among latent variables,
  while RLCD and LGES identify the whole underlying structure involving both observed and latent variables.

\subsection{Whether the Identifiability Theory Can be Extended to Non-Gaussian or Nonlinear Models?}
\label{sec:extend_to_nongaussian_nonlinear}
 The identifiability result can be extended to both linear non-Gaussian and certain kinds of nonlinear models.
(i) The key role of the score in our identifiability theory is to check whether the constraints on the observed covariance matrix is
 a subset of the constraints entailed by a candidate structure. For any structure, the constraints entailed by 
 in the non-Gaussian case is exactly the same as that of the Gaussian case.
  Therefore, the proposed identifiability theory still holds in the non-Gaussian scenario.
(ii) For certain kinds of nonlinearity, the proposed identifiability theory still works. For example, in Nonparanormal models where there exist smooth, monotonic transformations for each variable to transform variables to be jointly Gaussian, as the monotonic transformations can be identified up to trivial indeterminacy, we can still use the proposed score-based theory for structure identification.

{
\subsection{Further Discussion About  Error Propagation}
\label{sec:discussion_on_error_propagation}

Our score-based greedy search method is similar to a constraint-based method in the sense that both starts from a complete graph and deletes edges according to some criterion.  In this sense, why score-based methods are expected to suffer less from error propagation? 
The main reason lies in the difference of the used criteria. Score-based methods rely on MLE score and dimension, while constraint-based methods rely on statistical tests. We provide our further analysis as follows.

\textbf{Using Score as the Criterion Can Have Fewer Algorithm Steps.} The core theoretical basis of constraint-based methods and score-based methods are the same - checking whether the equality constraints on the observational distribution are aligned with the constraints entailed by a candidate graph.
Thus, roughly speaking, given a dataset the total number of constraints that need to be checked is the same for both score-based and constraint-based methods (if we want to guarantee asymptotic correctness).
However, constraint-based methods only examine one constraint in each step by a single test, while score-based methods can in essence examine multiple constraints together at the same time. More specifically, in score-based methods, we can introduce multiple constraints in a step and examine them together at the same time, by checking whether the likelihood is still maximal. 
As a consequence, score-based methods require fewer algorithm steps, which  mitigates the problem of error propagation. This is also aligned with our empirical observation that,
although the score-based GES and the constraint-based PC require basically  the same assumptions, 
in practice GES often has fewer steps to finish, and can often handle 10 times more variables than PC.

\textbf{Using Test as the Criterion Suffers More From Small Sample Sizes.}
Statistical tests (except for some permutation-based) reply on an approximation of the asymptotic null distribution of the test statistic to calculate the p-value for controlling the type-I error. When the sample size is small, e.g., N=100, the approximation (often based on the central limit theorem) could be far away from the true null distribution. As a consequence, with small sample size the type-I error in each step of the constraint-based methods cannot be properly controlled, let along the type-II error. 
This is also consistent with 
our empirical results: in Table 1 and Table 2, LGES still performs well with a very small sample size (N=100)
both with and without model mis-specification, while constraint-based methods such as RLCD and GIN does not work well in these scenarios.
}

\section{Limitations}
\label{sec:limitation}
One limitation of this work is that our theoretical results are based on the assumption of linear causal models. 
When data is not linear, we have also conducted experiments to see the performance and it can be shown that  our method still performs well.
Yet,  theoretical analysis and identifiability guarantee for the nonlinear case are to be developed and will be the focus of future work.

\section{Broader Impacts}
\label{sec:impacts}
The goal of this paper is to advance the field of machine learning. We do not see any potential negative societal impacts of the work.

\begin{figure}[t]
  \vspace{-0mm}
\setlength{\belowcaptionskip}{0mm}
    \subfloat[$\graph_1$.]
    {
    \centering
    \includegraphics[width=0.38\textwidth]{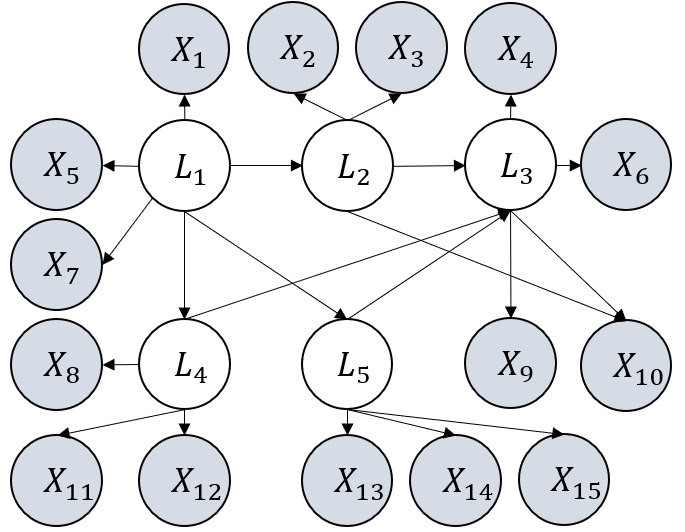}
    }
    \hfill
    \subfloat[$\graph_2$.]
    {
    \centering
    \includegraphics[width=0.38\textwidth]{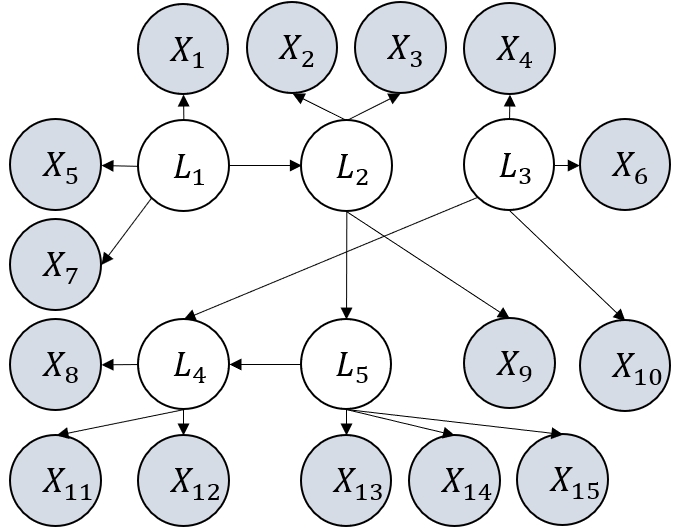}
    }
    \hfill
    \subfloat[$\graph_3$.]
    {
    \centering
    \includegraphics[width=0.38\textwidth]{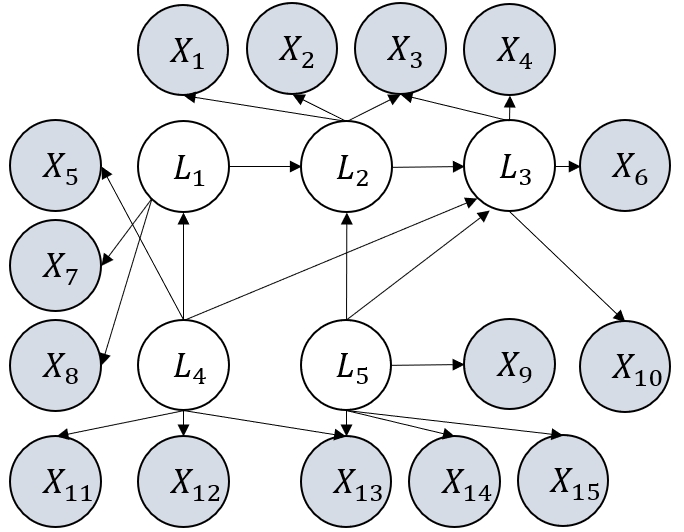}
    }
    \hfill
    \subfloat[$\graph_4$.]
    {
    \centering
    \includegraphics[width=0.38\textwidth]{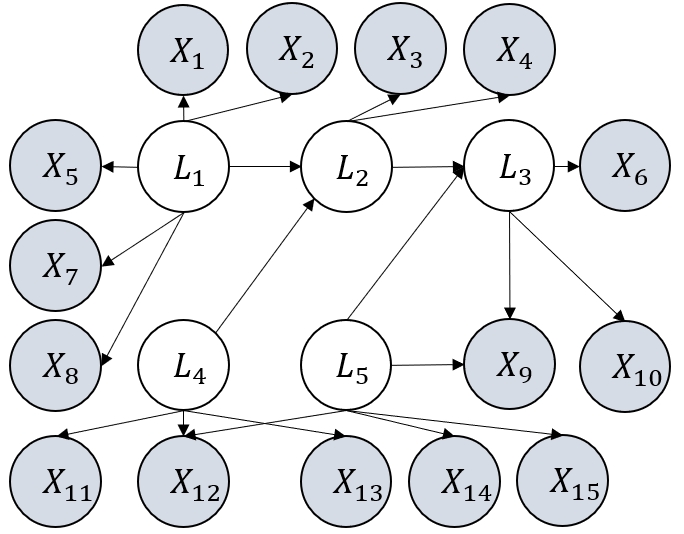}
    }
    \hfill
    \subfloat[$\graph_5$.]
    {
    \centering
    \includegraphics[width=0.38\textwidth]{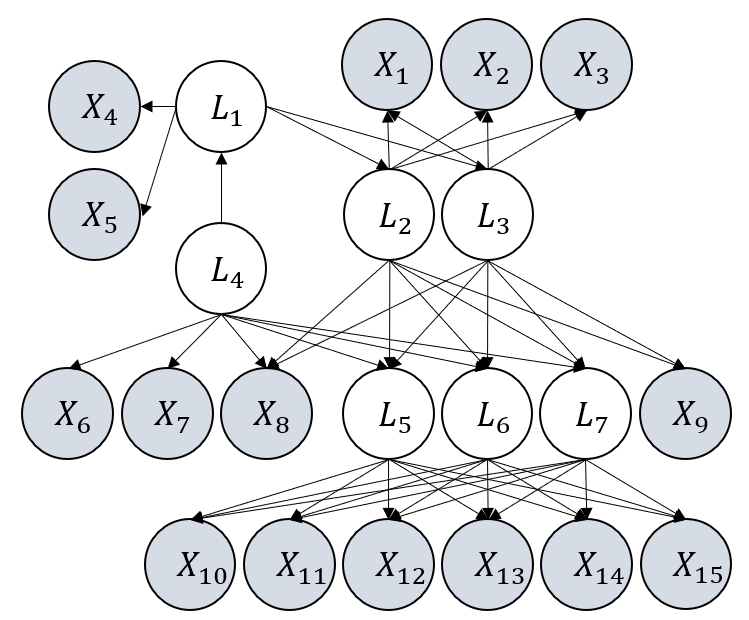}
    }
    \hfill
    \subfloat[$\graph_6$.]
    {
    \centering
    \includegraphics[width=0.38\textwidth]{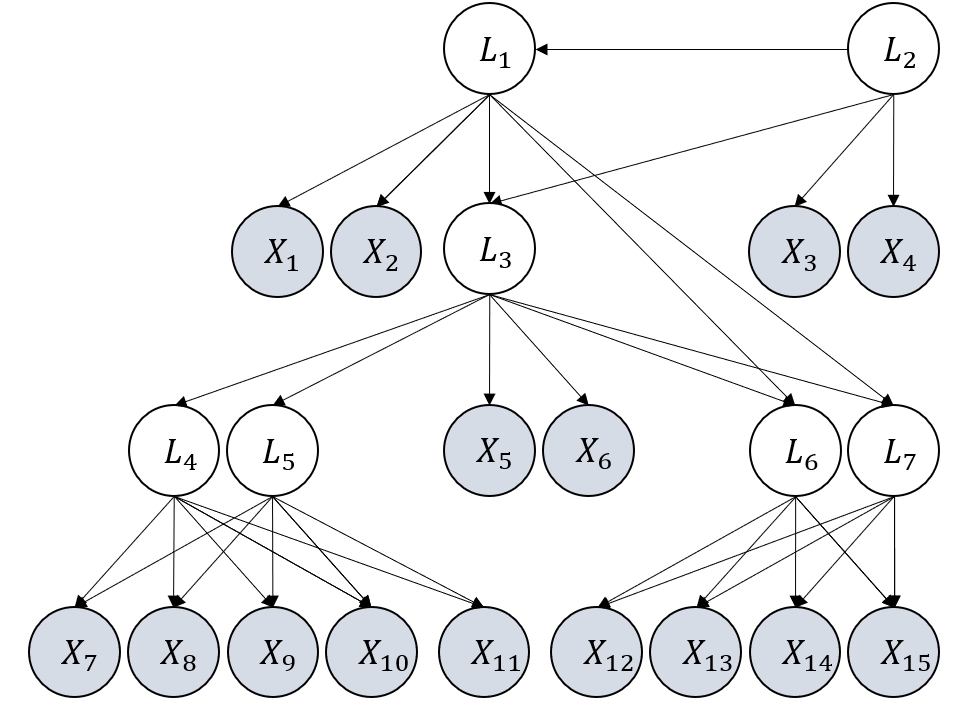}
    }
    \hfill
    \subfloat[$\graph_7$.]
    {
    \centering
    \includegraphics[width=0.35\textwidth]{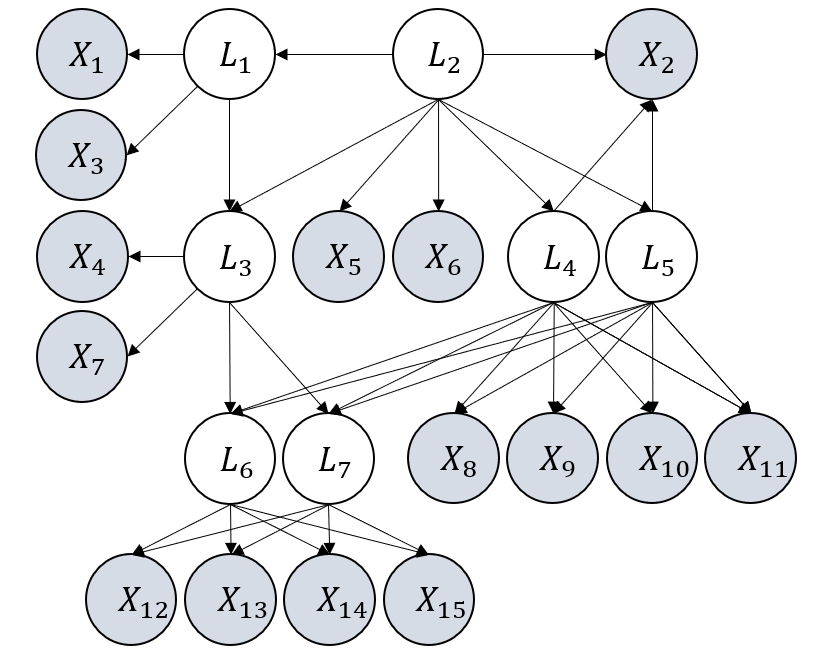}
    }
    \hfill
    \subfloat[$\graph_8$.]
    {
    \centering
    \includegraphics[width=0.38\textwidth]{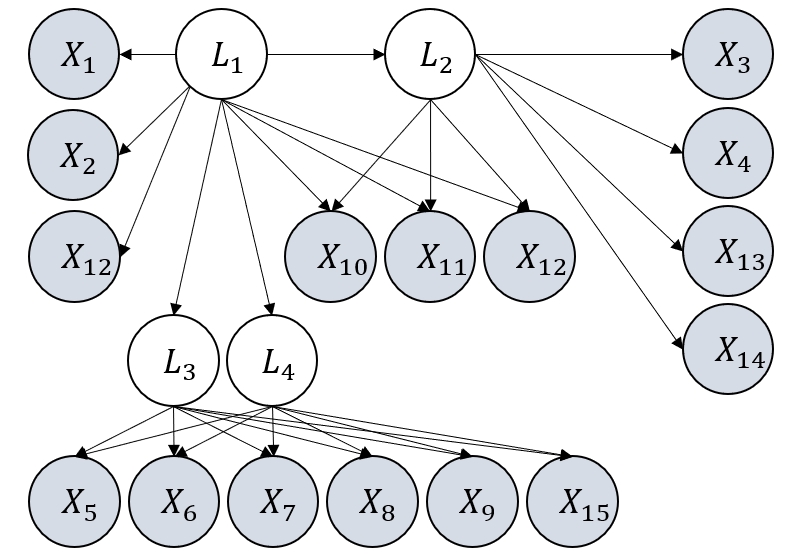}
    }
    \centering
    \vspace{-2mm}
\caption{
Examples of graphs considered in our experiments. They satisfy \cref{definition:gnfm}.
 }
  \label{fig:illustration_of_graphs}
  \vspace{-0mm}
\end{figure}

\end{document}